\documentclass[10pt]{article}
\usepackage{amsmath,amsfonts,amsthm,mathrsfs,url}

\usepackage{color}
\usepackage{fancyhdr}
\usepackage[top=2.3cm, bottom=2.5cm, left=2.1cm,
right=2.1cm]{geometry}
\usepackage{bm}
\usepackage{enumitem,amsthm,amsmath,amsfonts,amssymb}
\usepackage{float}
\usepackage{mathtools}
\usepackage{algorithmic}
\usepackage{algorithm}
\usepackage{xcolor}
\usepackage{graphicx}
\usepackage{subcaption}
\usepackage{tikz}
\usepackage{multirow}
\usepackage{multicol}
\usepackage{hyperref}
\usepackage{mathrsfs, fancyhdr}
\usepackage{multirow}
\usepackage{booktabs}

\newcommand{\Z}{\mathbb{Z}}

\newcommand{\N}{\mathbb{N}}

\newcommand{\D}{\mathcal{D}}
\newcommand{\E}{\mathbb{E}}

\newcommand{\R}{\mathbb{R}}

\newtheorem{cor}{Corollary}

\newtheorem{prop}{Proposition}
\newtheorem{rem}{Remark}

\def\bw{\mathbf{w}}
\def\O{\mathcal{O}}
\def\gep{\epsilon}

\def\X{\mathcal{X}}
\def\Y{\mathcal{Y}}
\def\W{\mathcal{W}}

\def\Z{\mathcal{Z}}
\def\N{\mathcal{N}}

\def\mbI{\mathbb{I}}
\def\EX{\mathbb{E}}
\def\S{\mathcal{S}}
\def\bw{\mathbf{w}}

\def\priv{\text{priv}}
\def\sgd{\text{SGD}}

\def\bb{\mathbf{b}}
\def\A{\mathcal{A}}
\def\gb{\beta}
\def\ga{\alpha}
\def\gd{\delta}
\def\gep{\epsilon}
\def\gD{\Delta}
\def\cR{\mathcal{R}}

\def\0{\mathbf{0}}

\def\proj{\text{Proj}}
\newcommand{\renyi}{R\'enyi}

\newtheorem{theorem}{Theorem}
\newtheorem{lemma}[theorem]{Lemma}
\newtheorem{proposition}[theorem]{Proposition}
\newtheorem{corollary}[theorem]{Corollary}
\theoremstyle{definition}
\newtheorem{definition}{Definition}

\newtheorem{example}{Example}
\theoremstyle{definition}

\newtheorem{remark}{Remark}

\def\begeqn{\begin{equation}}
\def\endeqn{\end{equation}}
\def\begth{\begin{theorem}}
\def\endth{\end{theorem}}
\def\begprop{\begin{proposition}}
\def\endprop{\end{proposition}}
\def\begcor{\begin{corollary}}
\def\endcor{\end{corollary}}
\def\begdef{\begin{definition}}
\def\enddef{\end{definition}}
\def\beglemm{\begin{lemma}}
\def\endlemm{\end{lemma}}
\def\begexm{\begin{example}}
\def\endexm{\end{example}}
\def\begrem{\begin{remark}}
\def\endrem{\end{remark}}

\def\begdef{\begin{definition}}
\def\enddef{\end{definition}}

\def\excess{\gep_{\text{risk}}}
\def\proj{\text{Proj}}

\begin{document}

\title{Differentially Private SGD with Non-Smooth Losses\thanks{Corresponding author: Yiming Ying.   Email: yying@albany.edu}}

\author{Puyu Wang$^{\dag}$, Yunwen Lei$^{\ddag}$, Yiming Ying$^*$ and Hai Zhang$^\dag$ \\
\\
$^{\dag}$ School of Mathematics, Northwest University, Xi'an, 710127, China\\
$^{\ddag}$ School of Computer Science, University of Birmingham, Birmingham B15 2TT, UK\\
$^{*}$Department of Mathematics and Statistics, State University of New York at Albany, \\
Albany, NY, 12222, USA}

\date{}

\maketitle

\begin{abstract} In this paper, we are concerned with  differentially  private {stochastic gradient descent (SGD)} algorithms in the setting of stochastic convex optimization (SCO). Most of the existing work requires the loss to be  Lipschitz continuous  and strongly smooth, and the model parameter to be uniformly bounded.  However, 
these assumptions are restrictive as many popular losses violate these conditions including the hinge loss for SVM, the absolute loss in robust regression, and even the least square loss in an  unbounded domain. We  significantly relax these restrictive assumptions  and establish privacy and generalization (utility) guarantees for private SGD algorithms using output and gradient perturbations  associated with non-smooth convex losses.   Specifically,  the loss function is relaxed to have an  $\alpha$-H\"{o}lder  continuous gradient (referred to as {\em $\alpha$-H\"{o}lder smoothness}) which instantiates the Lipschitz continuity ($\ga=0$) and  the strong smoothness ($\ga=1$). We prove that noisy  SGD with $\alpha$-H\"older smooth losses using gradient perturbation can guarantee $(\gep,\gd)$-differential privacy (DP) and attain optimal excess population risk $\O\Big(\frac{\sqrt{d\log(1/\delta)}}{n\epsilon}+\frac{1}{\sqrt{n}}\Big)$, up to logarithmic terms, with the gradient complexity  $ \O( n^{2-\ga\over 1+\ga}+ n).$ 
This shows an important trade-off  between $\alpha$-H\"older smoothness of the loss and the computational complexity for private SGD with statistically optimal  performance.  In particular, our results indicate that $\alpha$-H\"older smoothness  with $\ga\ge {1/2}$ is sufficient to guarantee $(\gep,\gd)$-DP of noisy SGD algorithms while achieving  optimal excess risk with the linear gradient complexity $  \O(n).$

\bigskip

\end{abstract}

\parindent=0cm

\section{Introduction}
Stochastic gradient descent (SGD) algorithms  are widely employed to train a wide range of machine learning (ML) models such as SVM, logistic regression, and deep neural networks. It is an iterative algorithm which replaces the true gradient on the entire training data by a randomized gradient estimated from a random subset (mini-batch) of the available data. As opposed to gradient descent algorithms, this reduces the computational burden at each iteration trading for a lower convergence rate  \cite{bousquet2008tradeoffs}. There is a large amount of work considering the optimization error (convergence analysis) of SGD and its variants in the linear case \cite{Bach,johnson2013accelerating,lacoste2012simpler,rakhlin2012making,shamir2013stochastic} as well as the general setting of reproducing kernel Hilbert spaces  \cite{Dieuleveut,lin2016optimal,Orabona,ying2008online, ying2017unregularized, smale2006online}.

At the same time, data collected often contain sensitive information such as individual records from schools and hospitals, financial records for fraud detection, online behavior from social media and genomic data from cancer diagnosis. Modern ML algorithms can explore the fine-grained information about data in order to make a perfect prediction which, however, can lead to privacy leakage \cite{carlini2019secret,shokri2017membership}. To a large extent, SGD algorithms have become the workhorse behind  the remarkable progress of ML and AI. Therefore, it is of pivotal importance for developing privacy-preserving SGD algorithms to protect the privacy of the data.  Differential privacy (DP) \cite{dwork2009differential,dwork2014algorithmic} has emerged as a well-accepted mathematical definition of privacy which ensures that an attacker gets roughly the same information from the dataset regardless of whether an individual is present or not. Its related technologies have been adopted by Google \cite{google-DP}, Apple \cite{apple-DP}, Microsoft \cite{microsoft-DP} and the US Census Bureau \cite{us-census-bureau-DP}.

In this paper, we are concerned with differentially private SGD algorithms in the setting of stochastic convex optimization (SCO). Specifically, let the input space $\X$ be a domain in some Euclidean space,  the output space $\Y\subseteq \R$, and $\Z = \X \times \Y.$ Denote the loss function by $\ell: \R^d \times \Z \mapsto [0, \infty)$ and assume, for any $z\in \Z$,  that $\ell(\cdot, z)$ is  convex  with respect to (w.r.t.) the first argument.   SCO aims to minimize the expected (population) risk, i.e. $ \cR(\bw) := \E_z [\ell(\bw,z)]$, where the model parameter $\bw$ belongs to a  (not necessarily bounded)  domain $\W\subseteq \R^d$, and  the expectation is taken w.r.t. $z$  according to 
a population  distribution $\D.$  While the population distribution is usually unknown, we have access to {a finite set of $n$ training data points} denoted by $S  = \{z_i \in\Z: i = 1,2,\ldots, n\}.$ It is  assumed to be independently  and identically distributed (i.i.d.) according to the distribution $\D$ on $\Z. $  In this context, one often  considers SGD algorithms to solve the Empirical Risk Minimization (ERM) problem defined by 
\[
\min_{\bw\in \W} \Bigl\{\cR_S(\bw): = \frac{1}{n}\sum_{i=1}^n \ell(\bw, z_i) \Bigr\}.\] 

For a randomized algorithm (e.g., SGD) $\A$ to solve the above ERM problem, let $\A(S)$ be the output of algorithm $\A$ based on the dataset $S$.  Then, its statistical generalization performance is measured by the excess (population) risk, i.e., the discrepancy between the expected risk $\cR(\A(S))$ and the least possible one in $\W$, which is defined by
\[
\excess(\A(S)) = \cR(\A(S))  -  \min_{\bw\in \W} \cR(\bw).
\]
Along this line,  there are a considerable amount of work \cite{wu2017bolt,bassily2019private,feldman2020private}  on analyzing the excess risk of private SGD algorithms in the setting of SCO.   However, most of such approaches often require two assumptions:  1) the loss $\ell$ is   $L$-Lipschitz and $\gb$-smooth; 2) the domain $\W$ is   uniformly bounded.  
These assumptions are very restrictive as many popular losses violate these conditions including the hinge loss $(1-y\bw^Tx)^q_+$ for $q$-norm soft margin SVM and the $q$-norm loss $|y-\bw^Tx|^q$ in regression with $1\le q\le 2.$ More specifically,   the work \cite{wu2017bolt} assumed the loss to be Lipschitz continuous and strongly smooth and showed that the private SGD algorithm with output perturbation can achieve $(\gep, \gd)$-DP and an excess risk rate   $\O(\frac{(d\log(1/\delta))^{1/4}}{\sqrt{n\epsilon}})$ when the gradient complexity (i.e. the number of computing gradients) $T=n$. The study \cite{bassily2019private} proved, under the same assumptions,  that  the  private SGD algorithm with gradient perturbation  can achieve an optimal excess risk rate $\O\Big(\frac{\sqrt{d\log(1/\delta)}}{n\epsilon}+ \frac{1}{\sqrt{n}}\Big)$ while guaranteeing its $(\gep, \gd)$-DP. To deal with the non-smoothness,  it used the Moreau envelope technique to smooth the loss function and got the optimal rate. However, the algorithm is  computationally inefficient with a  gradient complexity $\O\Big(n^{4.5}\sqrt{\epsilon}+ \frac{n^{6.5}\epsilon^{4.5}}{(d \log(1/\delta))^2}\Big).$ The work \cite{feldman2020private} improved the gradient complexity of the algorithm to $\O(n^2\log(\frac{1}{\delta}))$ by localizing the approximate minimizer of the population loss on each phase. Recently,  \cite{bassily2020stability} showed that a simple variant of noisy projected SGD yields the optimal rate with gradient complexity $\O(n^2)$.  However, it only focused on the Lipschitz continuous losses and assumed that the parameter domain $\W$ is bounded.

Our main contribution is to significantly  relax these restrictive assumptions  and to prove both privacy and generalization (utility) guarantees for private SGD algorithms  with non-smooth convex losses in both bounded and unbounded domains. 
Specifically, the loss function $\ell(\bw, z)$ is relaxed to have an  $\alpha$-H\"{o}lder  continuous gradient w.r.t. the first argument, i.e.,  there exists $L>0$ such that, for any $\bw, \bw'\in \W$ and any $z\in \Z$,  
\[
\|\partial\ell(\bw,z) - \partial\ell(\bw',z)\|_2 \le L \|\bw - \bw'\|_2^\ga,\]
where $\|\cdot\|_2$ denotes the Euclidean norm,  $\partial \ell(\bw,z)$ denotes a subgradient of $\ell$ w.r.t. the first argument. 
For the sake of notional simplicity,  we refer to this condition as {\em $\alpha$-H\"{o}lder smoothness} with parameter $L$. 
The smoothness parameter $\alpha\in [0,1]$ characterizes the smoothness of the loss function $\ell(\cdot, z)$.  The case of $\alpha =0$ corresponds to  the Lipschitz continuity of the loss $\ell$ while  $\alpha=1$ means its  strong smoothness. {This definition instantiates many non-smooth loss functions mentioned above. For instance, the hinge loss for $q$-norm soft-margin SVM  and $q$-norm loss for regression mentioned above with $q\in [1,2]$ are $(q-1)$-H\"older smooth.}  In particular, we prove that noisy  SGD with $\alpha$-H\"older smooth losses using gradient perturbation can guarantee $(\gep,\gd)$-DP and attain the optimal excess population risk $\O\Big(\frac{\sqrt{d\log(1/\delta)}}{n\epsilon}+\frac{1}{\sqrt{n}}\Big)$, up to logarithmic terms, with gradient complexity  $ \O( n^{2-\ga\over 1+\ga}+ n).$  This shows an important trade-off  between $\alpha$-H\"older smoothness of the loss and the computational complexity for private SGD in order to achieve statistically optimal  performance.  In particular, our results indicate that $\alpha$-H\"older smoothness  with $\ga\ge {1/2}$ is sufficient to guarantee $(\gep,\gd)$-DP of noisy SGD algorithms while achieving the optimal excess risk with linear gradient complexity $  \O(n).$  Table 1 summarizes the upper bound of the excess population risk, gradient complexity of the aforementioned algorithms in comparison to our methods.

Our key idea to handle general  H\"older smooth losses is to establish the approximate non-expansiveness of the gradient mapping, and the refined boundedness of the iterates of  SGD algorithms when domain $\W$ is unbounded.  This allows us to show the uniform argument stability \cite{liu2017algorithmic} of the iterates of SGD algorithms with high probability w.r.t. the internal randomness of the algorithm (not w.r.t. the data $S$), and consequently estimate the generalization error of differentially private SGD with non-smooth losses.

\begin{table}[H]\label{table:summary}
{\small{
\begin{center}
\begin{tabular}{ |c|ccccc| } 
 \hline
 \multirow{2}{*}{\textbf{Reference} }& \multirow{2}{*}{\textbf{Loss}} & \multirow{2}{*}{\textbf{ Method}}  & \multirow{2}{*}{ \textbf{Utility bounds}} &  \multirow{2}{*}{ \textbf{Gradient Complexity}} & \multirow{2}{*}{ \textbf{Domain}}  \\
 &&&&&\\
 \hline
 \multirow{3}{*}{\cite{wu2017bolt}} & \multirow{2}{*}{Lipschitz} & \multirow{3}{*}{ \textit{Output}}  &  \multirow{3}{*}{$\O\Big(\frac{(d \log(\frac{1}{\delta}))^{\frac{1}{4}}}{\sqrt{n\epsilon}} \Big)$}&  \multirow{3}{*}{$\O\big( n \big)$}  &  \multirow{3}{*}{bounded}\\
 ~& \multirow{2}{*}{\& smooth} &&&&\\
  &&&&&\\
 \hline
 \multirow{6}{*}{\cite{bassily2019private}} & \multirow{2}{*}{Lipschitz} & \multirow{3}{*}{ \textit{Gradient}}  &  \multirow{3}{*}{$\O\Big(\frac{\sqrt{d \log(\frac{1}{\delta})}}{ n \epsilon}+ \frac{1}{\sqrt{n}}\Big)$}&  \multirow{3}{*}{$\O\Big( n^{1.5} \sqrt{\epsilon} + \frac{(n\epsilon)^{2.5}}{ d \log(\frac{1}{\delta})} \Big)$}  &  \multirow{3}{*}{bounded}\\
 ~&  \multirow{2}{*}{\& smooth} &&&& \\  
 &&&&&\\
 \cline{2-6}
  ~ & \multirow{3}{*}{ Lipschitz} & \multirow{3}{*}{ \textit{Gradient}} &  \multirow{3}{*}{$\O\Big(\frac{\sqrt{d \log(\frac{1}{\delta})}}{ n \epsilon}+ \frac{1}{\sqrt{n}}\Big)$}&  \multirow{3}{*}{$\O\Big( n^{4.5} \sqrt{\epsilon} + \frac{n^{6.5}\epsilon^{4.5}}{( d \log(\frac{1}{\delta}))^2} \Big)$}  &  \multirow{3}{*}{bounded} \\ 
  &&&&&\\
 &&&&&\\
 \hline
 \multirow{6}{*}{\cite{feldman2020private}} & \multirow{2}{*}{Lipschitz} & \multirow{3}{*}{\textit{Phased Output}} & \multirow{3}{*}{$\O\Big(\frac{\sqrt{d \log(\frac{1}{\delta})}}{ n \epsilon}+ \frac{1}{\sqrt{n}}\Big)$}&  \multirow{3}{*}{$\O\big(n  \big)$}  &  \multirow{3}{*}{bounded  } \\
 &  \multirow{2}{*}{\& smooth} &&&&  \\ 
 &&&&&\\
 \cline{2-6}
 ~ & \multirow{3}{*}{ Lipschitz}  & \multirow{3}{*}{\textit{Phased ERM}} & \multirow{3}{*}{$\O\Big(\frac{\sqrt{d \log(\frac{1}{\delta})}}{ n \epsilon}+ \frac{1}{\sqrt{n}}\Big)$}&  \multirow{3}{*}{$\O\big( n^2  \log(\frac{1}{\delta}) \big)$}  &  \multirow{3}{*}{ bounded  } \\ 
 &&&&&\\
 &&&&&\\
 \hline
 \multirow{3}{*}{\cite{bassily2020stability}} & \multirow{3}{*}{ Lipschitz} & \multirow{3}{*}{\textit{Gradient}}  &  \multirow{3}{*}{$\O\Big(\frac{\sqrt{d \log(\frac{1}{\delta})}}{ n \epsilon}+ \frac{1}{\sqrt{n}}\Big)$}&  \multirow{3}{*}{$\O\big( n^2 \big)$}  &  \multirow{3}{*}{ bounded } \\ 
 &&&&&\\
 &&&&&\\
 \hline
 \multirow{10}{*}{Ours} & \multirow{2}{*}{ $\alpha$-H\"older }  &  \multirow{3}{*}{\textit{Output}}&  \multirow{3}{*}{$\O\Big(  \frac{(d\log(\frac{1}{\delta}))^{\frac{1}{4}} \sqrt{\log(\frac{n}{\delta})} }{\sqrt{n \epsilon }}\Big)$} &  \multirow{3}{*}{$\O\big(n^{\frac{2-\alpha}{1+\alpha}} + n \big)$}&  \multirow{3}{*}{ bounded  } \\ 
 ~& \multirow{2}{*}{ smooth} &&&&\\
 &&&&&\\
 \cline{2-6}
 ~&\multirow{3}{*}{$\alpha$-H\"older  }&\multirow{4}{*}{\textit{Output}}&\multirow{4}{*}{$\O\Big(\frac{\sqrt{d\log(\frac{1}{\delta})} \log(\frac{n}{\delta}) }{ n^{\frac{2}{3+\alpha}} \epsilon  }+ \frac{ \log(\frac{n}{\delta}) }{n^{\frac{1}{3+\alpha}}} \Big)$}&\multirow{4}{*}{$\O\big(n^{\frac{-\alpha^2-3\alpha+6}{(1+\alpha)(3+\alpha)}}+n\big)$}&  \multirow{4}{*}{ unbounded  } \\
  ~& \multirow{3}{*}{ smooth} &&&&\\
  &&&&&\\
  &&&&&\\
 \cline{2-6}
 ~&\multirow{3}{*}{$\alpha$-H\"older }&\multirow{4}{*}{\textit{Gradient}}&\multirow{4}{*}{$\O\Big(\frac{\sqrt{d\log(\frac{1}{\delta}) }  }{n\epsilon} + \frac{ 1}{ \sqrt{n} }  \Big)$}&\multirow{4}{*}{$\O\big(n^{
\frac{2-\alpha}{1+\alpha}}+n\big)$}&  \multirow{4}{*}{ bounded  } \\
  ~& \multirow{3}{*}{ smooth} &&&&\\
  &&&&&\\
  &&&&&\\
 \hline 
\end{tabular}
\smallskip
\caption{Comparison of different $(\epsilon,\delta)$-DP algorithms.  We report the method, utility (generalization) bound, gradient complexity and parameter domain for three types of convex losses, i.e. Lipschitz, Lipschitz and smooth, and $\alpha$-H\"older smooth. Here \textit{Output}, \textit{Gradient}, \textit{Phased Output} and \textit{Phased ERM} denote output perturbation which adds Gaussian noise to the output of non-private SGD, gradient perturbation which adds Gaussian noise at each SGD update, phased output perturbation and phased ERM output perturbation \cite{feldman2020private}, respectively. The gradient complexity is the total number of computing the gradient on one datum in the algorithm.  }
\end{center}
}}\vspace*{-0.05in}
\end{table}

\smallskip 
\noindent \textbf{Organization of the Paper.} The rest of the paper is organized as follows.   The formulation of SGD algorithms and the main results are given in Section \ref{sec:formulation}. We provide the proofs in Section \ref{sec:proofs} and conclude the paper in Section \ref{sec:conclusion}.

\section{Problem Formulation and Main Results}\label{sec:formulation}
\subsection{Preliminaries}\label{sec:preliminaries}

Throughout the paper, we assume that the loss function $\ell: \W\times \Z\rightarrow \R$ is convex  w.r.t. the first argument, i.e., for any $ z  \in \Z$ and $\bw,\bw'\in\mathcal{W}$, there holds $ \ell(\bw, z)\ge \ell(\bw', z) + \langle \partial \ell(\bw', z) , \bw-\bw'  \rangle  $ where $\partial \ell(\bw', z)$ denotes a subgradient of $\ell(\cdot, z)$ in the first argument.  
We restrict our attention to the (projected) stochastic gradient descent algorithm which is defined as below. 
\begin{definition}[Stochastic Gradient Descent]\label{def-SGD-updaterule}

Let $\W \subseteq \R^d$ be convex,  {$T$ denote the number of iterations}, and $\proj_{\W}$ denote the projection to $\W$.
Let $\bw_1=\mathbf{0} \in \R^d$ be an initial point, and $\{\eta_t\}_{t=1}^{T-1}$ be a sequence of positive step sizes. At step $t\in \{1,\ldots, T-1\}$, the update rule of (projected) stochastic gradient decent   is given by 
\begin{equation}\label{eq:pro-SGD-updaterule}
\bw_{t+1}=\proj_{\W} \big(\bw_t-\eta_t \partial \ell(\bw_t,z_{i_t})\big),
\end{equation}
where $\{i_t\}$ is uniformly drawn from $[n]:=\{1,2,\ldots,n\}$.  
When $\W = \R^d$, then \eqref{eq:pro-SGD-updaterule} is reduced to
$
\bw_{t+1}=\bw_t-\eta_t \partial \ell(\bw_t,z_{i_t}).$
\end{definition}

For a randomized learning algorithm $\A: \Z^n \rightarrow \W$, let $\A(S)$ denote the  model produced by running $\A$ over the training dataset $S$. We say two datasets $S$ and $S'$ are  \textit{neighboring} datasets, denoted by $S\simeq S'$, if they differ by a single datum. {We consider the following high-probabilistic version of the uniform argument stability (UAS), 
which is an extension of the UAS in expectation \cite{liu2017algorithmic}. }
\begin{definition}[Uniform argument stability]\label{def:uniform-stability}
We say an algorithm $\A$ has $\Delta_\A$-UAS with probability at least $1-\gamma$ \big($\gamma\in(0,1)$\big) if
$$\mathbb{P}_\A (\sup_{S\simeq S'} \delta_\A(S,S')\ge \Delta_\A)\le \gamma,$$
where  
$ \delta_{\mathcal{\A}}(S,S'):= \|\mathcal{A}(S)-\mathcal{A}(S')\|_2.$
\end{definition}
{
We will use UAS to study generalization bounds with high probability.
In particular, 
the following lemma as a straightforward extension of Corollary 8 in \cite{bousquet2019sharper}  establishes the relationship between UAS and generalization errors. The proof is given in the Appendix for completeness.} \begin{lemma}\label{lem:generror-high-probability}
Suppose $\ell$ is nonnegative, convex and $\alpha$-H\"older smooth with parameter $L$.  Let $M_0=\sup_{z\in \Z} \ell(0,z)$ and $M=\sup_{z\in\Z} \|\partial \ell(0, z)\|_2 $.
Let $\A$ be a randomized algorithm with the output of $\A$ bounded by $G$ and
 $$\mathbb{P}_\A (\sup_{S\simeq S'}\delta_\A(S,S')\ge \Delta_\A)\le \gamma_0.$$  
Then there exists a constant $c>0$ such that for any distribution $\D$ over $\Z$ and any $\gamma \in (0,1)$, there holds
$$\mathbb{P}_{\S \sim \D^n, \A}\biggl[|\cR(\mathcal{A(S)})-\cR_S(\mathcal{A(S)})|\ge c\bigg((M+ LG^\alpha  ) \Delta_\A \log(n)\log(1/{\gamma})+\big( M_0+ (M+LG^\alpha)G\big)   \sqrt{n^{-1}\log(1/\gamma)} \bigg)\biggr] \le \gamma_0+\gamma.$$
\end{lemma}
Differential privacy  \cite{dwork2006calibrating} is a \textit{de facto} standard privacy measure for a randomized algorithm $\A.$ 
\begin{definition}[Differential Privacy]\label{def:DP}
We say a randomized algorithm $\A$ satisfies $(\gep, \delta)$-DP if, for any two neighboring datasets $S$ and $S'$  and any event $E$ in the output space of $\A$, there holds 
\[
\mathbb{P}(\A (S)\in E) \le e^\gep \mathbb{P}(\A(S')\in E)+ \gd.
\]
In particular,  we call it satisfies $\gep$-DP if $\gd=0$. 
\end{definition} 
We also need the following concept called  $\ell_2$-sensitivity.
\begin{definition}[$\ell_2$-sensitivity]\label{def:sensitivity}
The $\ell_2$-sensitivity of a function (mechanism) $\mathcal{M}:\Z^n \rightarrow \W$ is defined as 
$
\gD  = \sup_{S\simeq S'} \|\mathcal{M}(S) - \mathcal{M}(S')\|_2,
$ where $S$ and $S'$ are neighboring datasets. 
\end{definition}

A basic mechanism to obtain $(\gep, \delta)$-DP from  a given function  $\mathcal{M}:\Z^n \rightarrow \W$ is to add a random  noise from a Gaussian distribution $\mathcal{N}(0, \sigma^2\mathbf{I}_d )$ where $\sigma$ is proportional to its $\ell_2$-sensitivity. This mechanism is often referred to as Gaussian mechanism as stated in the following lemma. 

\begin{lemma}[\cite{dwork2014algorithmic}]\label{lem:gaussian-noise}
Given a  function $\mathcal{M}: \Z^n \rightarrow \W$ with the $\ell_2$-sensitivity $\gD $ and a dataset $S\subset \Z^n$, and assume that $\sigma \geq \frac{\sqrt{2\log(1.25/\delta)}\Delta }{\epsilon}$. The following Gaussian mechanism yields $(\gep, \gd)$-DP: 
\[\label{eq:gausian-mech}
\mathcal{G}(S,\sigma) := \mathcal{M}(S) + \bb, ~~\bb \sim \mathcal{N}(0, \sigma^2\mathbf{I}_d), 
\]
where $\mathbf{I}_d$ is the identity matrix in $\mathbb{R}^{d\times d}$.  
\end{lemma}

Although the concept of $(\epsilon,\delta)$-DP is widely used in privacy-preserving methods, its composition and subsampling amplification results are relatively loose, which are not suitable for  iterative SGD algorithms. Based on the \renyi \ divergence, the work \cite{mironov2017renyi} proposed  \renyi \  differential privacy (RDP) as a relaxation of DP to achieve tighter analysis of composition and amplification mechanisms. 
{
\begin{definition}[RDP \cite{mironov2017renyi}]\label{def:RDP}
For $\lambda > 1$, $\rho > 0$, a randomized mechanism $\A$ satisfies $(\lambda, \rho)$-RDP, if,  for all neighboring datasets $S$ and $S'$, we have 
     $$ D_{\lambda}\big(\A(S)\parallel \A(S')\big):= \frac{1}{\lambda-1}\log \int  \Big( \frac{ P_{\A(S)}(\theta) }{ P_{\A(S')}(\theta) }  \Big)^\lambda    d P_{\A(S')}(\theta) \le \rho,$$
    where $P_{\A(S)}(\theta)$ and $P_{\A(S')}(\theta)$ are the density of $\A(S) $ and $\A(S')$, respectively. 
\end{definition}
}
 
{As $\lambda \rightarrow \infty$, RDP reduces to $\epsilon$-DP, i.e., $\A$ satisfies $\epsilon$-DP if and only if $D_{\infty}\big(\A(S)||\A(S')\big)\le \epsilon$ for any neighboring datasets $S$ and $S'$. 
Our analysis requires the introduction of several lemmas on useful properties of RDP listed below. 
}

{First, we introduce the privacy amplification of RDP by uniform subsampling, which  is  fundamental to establish privacy guarantees of noisy SGD algorithms. In general, a uniform subsampling scheme first draws a subset with size $pn$ uniformly at random with a subsampling rate $p \le 1$,  and then applies a known randomized mechanism to the subset. 
}
\begin{lemma}[\cite{liang2020exploring}]\label{lem:uniform} Consider a function $\mathcal{M}:  \Z^n\rightarrow \W$ with the $\ell_2$-sensitivity $\Delta$,  and a dataset $S\subset\Z^n$.  
{The Gaussian mechanism $\mathcal{G}(S,\sigma)=\mathcal{M}(S)+\bb$, where $\bb\sim \mathcal{N}(0,\sigma^2\mathbf{I}_d)$, } applied to  a subset of samples that are drawn uniformly without replacement with subsampling rate $p$ satisfies 
$(\lambda,3.5p^2\lambda \Delta^2/\sigma^2)$-RDP  if $\sigma^2\geq 0.67 \Delta^2$ and $\lambda-1 \leq \frac{2\sigma^2}{3\Delta^2} \log \big(\frac{1}{\lambda p (1+ \sigma^2/\Delta^2)} \big)$. 
\end{lemma}

{ 
The following adaptive composition theorem of RDP establishes the privacy of a composition of several adaptive mechanisms in terms of that of individual mechanisms. We say a sequence of mechanisms $ (\A_1,\ldots,\A_k)  $ are chosen adaptively if $\A_i$ can be chosen based on the outputs of the previous mechanisms $\A_1(S),\ldots,\A_{i-1}(S)$ for any $i\in[k]$.  

\begin{lemma}[Adaptive Composition of RDP \cite{mironov2017renyi}]\label{lem:composition_RDP}
If a mechanism $\A$ consists of a sequence of  adaptive mechanisms $(\A_1,\ldots,\A_k)$ with $\A_i$ satisfying  $(\lambda, \rho_i)$-RDP, $i\in[k]$, then $\A$ satisfies $(\lambda, \sum_{i=1}^k \rho_i)$-RDP. 
\end{lemma}
}
{ 
Lemme \ref{lem:composition_RDP} tells us that  the derivation of the privacy guarantee for a composition mechanism is simple and direct.   This is the underlying reason that we adopt RDP in our subsequent privacy analysis. The following lemma allows us to further convert RDP back to  $(\epsilon,\delta)$-DP.}
\begin{lemma}[From RDP to $(\epsilon,\delta)$-DP \cite{mironov2017renyi}]\label{lemma:RDP_to_DP}
	If a randomized mechanism $\mathcal{A}$ satisfies $(\lambda,\rho)$-RDP, then $\mathcal{A}$ satisfies $(\rho+\log(1/\delta)/(\lambda-1),\delta)$-DP for all $\delta\in(0,1)$.
\end{lemma}
{
The following lemma shows that a post-processing procedure always preserves   privacy.
\begin{lemma}[Post-processing \cite{mironov2017renyi}]\label{lemma:post-processing}
Let  $\A: \Z^n \rightarrow \W_1 $  satisfy $(\lambda, \rho)$-RDP  and $f: \W_1 \rightarrow \W_2$ be an arbitrary function. Then $f \circ \A : \Z^n \rightarrow \W_2$ satisfies $(\lambda, \rho)$-RDP.     
\end{lemma}
}

\subsection{Main Results}\label{sec:main-results}
{We present our main results here. First, we state a key bound of UAS for SGD   when $\W\subseteq \R^d$ and the loss function is  $\ga$-H\"older smooth. Then, we propose two privacy-preserving SGD-type algorithms using output and gradient perturbations, and present the corresponding privacy and generalization (utility) guarantees. The utility guarantees in terms of the excess risk typically rely on two main   errors:  optimization errors and  generalization errors, as shown soon in \eqref{eq:err-decomp-out} and \eqref{eq:Gradient-error-decom2} for the algorithms with output  and gradient perturbations, respectively. 
We will apply techniques in optimization theory to handle the optimization errors \cite{Nem}, and the concept of UAS \cite{bousquet2002stability,hardt2016train,liu2017algorithmic}, which was given in Definition \ref{def:uniform-stability}  in Subsection \ref{sec:preliminaries}, to estimate the generalization errors. }

\subsubsection{UAS bound of SGD with  Non-Smooth Losses}

We begin by stating the key result on the distance between two iterate trajectories produced  by  SGD on neighboring datasets.
Let 
\begin{equation}\label{alpha-1}
  c_{\ga,1}=\begin{cases}
     (1+1/\ga)^{\frac{\ga}{1+\ga}}L^{\frac{1}{1+\ga}}, & \mbox{if } \ga\in (0,1] \\
                 M+L, & \mbox{if } \ga=0.
               \end{cases}
\end{equation} and $c_{\ga,2}=\sqrt{\frac{1-\ga}{1+\ga}}(2^{-\ga}L)^\frac{1}{1-\ga}$, where $M=\sup_{z\in \Z}\|\partial \ell(0,z)\|_2$.  In addition, define $C_\ga=\frac{1-\ga}{1+\ga}  c^{\frac{2(1+\ga)}{1-\ga}}_{\ga,1} \big( \frac{\ga}{1+\ga}\big)^{\frac{2\ga}{1-\ga}}+ 2\sup_{z\in \Z} \ell(0;z)$. Furthermore, let $\mathcal{B}(0, r)$ denote the Euclidean ball of radius $r>0$ centered at $0 \in \R^d$. Without loss of generality, we assume $\eta> 1/T$.

\begin{theorem}\label{thm:difference-w}
Suppose that the loss function $\ell$ is convex and $\alpha$-H\"older smooth with parameter $L$.  
Let $\A$ be the SGD with $T$ iterations and  $\eta_t=\eta < \min \{1, 1/L \} $, and $\bar{\bw}=\frac{1}{T}\sum_{t=1}^T \bw_t$ be the output produced by $\A$. Further, let  $c_{\gamma,T}=\max\Big\{\big(3n\log(n/\gamma)/T\big)^{\frac{1}{2}}, 3n\log(n/\gamma)/T \Big\}$.  
\begin{enumerate}[label=({\alph*})]
\item If $\ell$ is nonnegative and $\W=\R^d$, then, for any $\gamma \in (0,1)$, there holds
\begin{equation*}
\mathbb{P}_\A\Big( \sup_{S\simeq S'}  \delta_{\A}(S,S') \ge \Delta_{SGD}(\gamma) \Big) \le \gamma,
\end{equation*}
where $\Delta_{SGD}(\gamma)=\Big(e\big(c^2_{\alpha,2}T\eta^{\frac{2}{1-\alpha}}+4\big(M+L (C_\ga T \eta)^{\frac{\ga}{2}}\big)^2\eta^2\Big(1+\frac{T}{n}(1+c_{\gamma,T})\Big)\frac{T}{n}(1+c_{\gamma,T})\big)\Big)^{1/2}$. 
\item If $\W \subseteq \mathcal{B}(0,R)$ with $R> 0$, then, for any $\gamma \in (0,1)$, there holds
\begin{equation*}
 \mathbb{P}_\A\Big( \sup_{S\simeq S'} \delta_{\A}(S,S')   \ge \Tilde{\Delta}_\sgd(\gamma) \Big) \le \gamma,
\end{equation*}
where $\Tilde{\Delta}_\sgd(\gamma)=\Big(e\big(c^2_{\alpha,2}T\eta^{\frac{2}{1-\alpha}}+4\big(M+LR^\alpha \big)^2\eta^2\Big(1+\frac{T}{n}(1+c_{\gamma,T})\Big)\frac{T}{n}(1+c_{\gamma,T})\big)\Big)^{1/2}$.
\end{enumerate}
\end{theorem}
\begin{remark}Under the reasonable assumption of $T\ge n$, we have $c_{\gamma,T}=\O(\log(n/\gamma))$. Then $\Delta_\sgd(\gamma)=\O\Big(  \sqrt{T}\eta^{\frac{1}{1-\alpha}}+\frac{(T\eta)^{1+\alpha/2}\log( n /\gamma)}{n} \Big)$ and $\tilde{ \Delta}_\sgd(\gamma)=\O\Big( \sqrt{T}\eta^{\frac{1}{1-\alpha}}+\frac{  T \eta  \log( n /\gamma)}{n} \Big)$. In addition, if $\ell$ is strongly smooth, i.e., $\alpha=1$, the first term in the UAS bounds tends to $0$ under the typical assumption of $\eta< 1$. In this case we have $\Delta_\sgd(\gamma)=\O\Big(\frac{\big(T\eta\big)^{3/2}\log( n /\gamma)}{n}\Big)$ and $\tilde{ \Delta}_\sgd(\gamma)=\O\Big(\frac{  T \eta   \log( n /\gamma) }{n}\Big)$. The work \cite{bassily2020stability} established the high probability upper bound of the  random variable of the argument stability  $\delta_{SGD}$ in the order of $\O(\sqrt{T} \eta + \frac{T\eta}{n})$ for Lipschitz continuous losses under an additional assumption  $\gamma \ge \exp(-n/2)$. 
{Our result gives the upper bound of $\sup_{S\simeq S'} \delta_{SGD}(S,S')$ in the order of $\O(\sqrt{T}\eta+\frac{T\eta \log(n/\gamma)}{n})$ for any $\gamma \in (0,1)$ for the case of $\alpha=0$.} 
The work \cite{hardt2016train} gave the bound of $\O({T \eta}/{n})$ in expectation for Lipschitz continuous and smooth loss functions. As a comparison, our stability bounds are stated with high probability and do not require the Lipschitz condition. Under a further Lipschitz condition, our stability bounds actually recover the bound $\O({T \eta}/{n})$ in \cite{hardt2016train} in the smooth case. Indeed, both the term $\big(M+(C_\ga T \eta)^{\frac{\ga}{2}}\big)^2$  and the term $\big(M+L R^\alpha \big)^2$ are due to controlling the magnitude of gradients, and can be replaced by $L^2$ for $L$-Lipschitz losses. \end{remark}

\subsubsection{Differentially Private SGD with Output Perturbation }\label{sec:output}

\begin{algorithm}[t]
\begin{algorithmic}[1]
\caption{Differentially Private SGD with Output perturbation (\texttt{DP-SGD-Output})}\label{alg1}
\STATE{\bf Inputs:}  Data $S= \{z_{i}\in \Z: i=1,\ldots, n\}$,  $\ga$-H\"older smooth loss $\ell(\bw,z)$ with parameter $L$, the convex set $\W$, step size $\eta$, number of iterations $T$, and privacy parameters $\gep$, $\gd$
\STATE{\bf Set:}  $\bw_1=\0$ 
\FOR { $t=1$ to $T$ } 
\STATE{Sample $i_t\sim \text{Unif}([n])$} 
\STATE{$\bw_{t+1}=\proj_{\W}(\bw_t-\eta \partial \ell(\bw_t;z_{i_t}))$}
\ENDFOR
\IF{$\W=\R^d$}
\STATE{let $\Delta=\Delta_\sgd(\delta/2)$ }
\ELSIF{  $\W \subseteq \mathcal{B}(0,R)$  }
\STATE{let $\Delta=\Tilde{\Delta}_\sgd(\delta/2)$  }
\ENDIF
\STATE{\bf Compute:}
 $\sigma^2=\frac{2\log(2.5/\delta)\Delta^2}{\epsilon^2}$
\STATE  {\bf return:}   ${\bw}_\priv = \frac{1}{T}\sum_{t=1}^T\bw_t+ \bb$ where  $\bb\sim \N(0,\sigma^2 \mathbf{I}_d) $
\end{algorithmic}
\end{algorithm}

 {Output perturbation} \cite{chaudhuri2011differentially,dwork2006calibrating} is a common approach to achieve $(\gep,\gd)$-DP.  The main idea is to  add a random noise $\bb$ to the output of the SGD algorithm, where $\bb$ is randomly sampled from the Gaussian distribution with mean $0$ and variance proportional to the $\ell_2$-sensitivity of SGD. In Algorithm~\ref{alg1}, we  
 propose the private SGD algorithm with output perturbation for non-smooth losses in both bounded domain $\W \subseteq \mathcal{B}(0,R)$ and unbounded domain $\W=\R^d$. The difference in these two cases is that we add random noise with different variances according to the sensitivity analysis of SGD stated in Theorem \ref{thm:difference-w}. In the sequel, we  present the privacy and utility guarantees for Algorithm~\ref{alg1}. 
 
\begin{theorem}[Privacy guarantee]\label{thm:dpSGD-output-Gaussian} 
Suppose that the loss function $\ell$ is convex, nonnegative and $\alpha$-H\"older smooth with parameter $L$. Then Algorithm~\ref{alg1} (\texttt{DP-SGD-Output}) satisfies $(\gep,\gd)$-DP. 
\end{theorem}

{ 
According to the definitions, 
the $\ell_2$-sensitivity of SGD is identical to the UAS of SGD: $\sup_{S\simeq S'}\delta_{SGD}(S,S')$. 
}
In this sense,  the proof of Theorem \ref{thm:dpSGD-output-Gaussian} directly follows from Theorem \ref{thm:difference-w} and  Lemma \ref{lem:gaussian-noise}. For completeness, we include the detailed proof in Subsection~\ref{subsec:output}.

Recall that the empirical risk is defined by $\cR_S(\bw)=\frac{1}{n}\sum_{i=1}^n \ell(\bw,z_i)$, and the population risk is $\cR(\bw)=\EX_z[\ell(\bw,z)]$. Let $\bw^{*}\in \arg\min_{\bw\in \W} \cR(\bw)$ be the one with the best prediction performance over $\W$. We use the notation $B \asymp \tilde{B}$ if there exist constants $c_1, c_2> 0$ such that $c_1  \tilde{B}< B\le c_2\tilde{B}$.  Without loss of generality, we always assume $\|\bw^*\|_2\geq1$.

\begin{theorem}[Utility guarantee for unbounded domain]\label{thm:output-utility}
Suppose the loss function $\ell$ is nonnegative, convex  and $\alpha$-H\"older smooth with parameter $L$.  Let $\bw_{\priv}$ be the output produced by  Algorithm~\ref{alg1} with $\W=\R^d$ and  $\eta= n^{\frac{1}{3+\alpha}}/\big(T (\log(\frac{1}{\gamma}))^{\frac{1}{3+\alpha}} \big)$. Let $T\asymp n^{\frac{-\alpha^2-3\alpha+6}{(1+\alpha)(3+\alpha)}} $ if $0\le\alpha < \frac{\sqrt{73}-7}{4}$, and $T\asymp n$ else. Then, for any $\gamma \in (4\max\{\exp(-d/8),  \delta\},1)$, with probability at least $1-\gamma$ over the randomness in both the sample and the algorithm, there holds
\begin{equation*}
    \cR(\bw_{\priv})-\cR(\bw^{*})
    =  \|\bw^{*}\|_2^2\cdot  \O\bigg( \frac{\sqrt{d\log(1/\delta)} {\log(n/\delta)}  }{ (\log(1/\gamma))^{\frac{1+ \alpha}{4(3+\alpha)}} n^{\frac{2}{3+\alpha}} \epsilon} + \frac{\log(n)\big(\log(1/\gamma)\big)^{\frac{2}{ 3+\alpha }}
    {\log(n/\delta)}}{n^{\frac{1}{3+\alpha}}}\bigg).
\end{equation*}
\end{theorem}

\noindent 
To examine the excess population risk $\cR(\bw_{\priv})-\cR(\bw^{*})$, we use the following error decomposition: 
\begin{equation}\label{eq:err-decomp-out}
    \cR(\bw_{\priv})-\cR(\bw^{*})=[\cR(\bw_{\priv})-\cR(\Bar{\bw})]+[\cR(\Bar{\bw})-\cR_S(\Bar{\bw})]
    +[\cR_S(\Bar{\bw})-\cR_S(\bw^{*})] +[\cR_S(\bw^{*})-\cR(\bw^{*})],
\end{equation}
where $\Bar{\bw}=\frac{1}{T}\sum_{t=1}^T\bw_t$ is the output of non-private SGD. The first term is due to the added noise $\bb$, which can be estimated by the Chernoff bound for Gaussian random vectors. The second term is the generalization error of SGD, which can be handled by the stability analysis.  The third term is an optimization error and can be controlled by standard techniques in optimization theory. Finally, the last term can be bounded by $\O(1/\sqrt{n})$ by Hoeffding  inequality. The proof of Theorem~\ref{thm:output-utility} is given in Subsection~\ref{subsec:output}.

Now, we turn our attention to the utility guarantee for the case with a  bounded domain. 
\begin{theorem}[Utility guarantees for bounded domain]
\label{thm:output-utility-proj}
If the loss function $\ell$ is nonnegative, convex and $\alpha$-H\"older smooth with parameter $L$.  Let $\bw_{\priv}$ be the output produced by  Algorithm~\ref{alg1} with  $\W\subseteq \mathcal{B}(0,R)$. Let $T \asymp n^{
\frac{2-\alpha}{1+\alpha}} $ if $\alpha < \frac{1}{2}$,  $T\asymp n$ else, and choose $\eta=1 / \Big(T \max\Big\{ \frac{ \sqrt{\log(n/\delta) \log(n) \log(1/\gamma)} }{\sqrt{n}}, \frac{    \big(d\log(1/\delta)\big)^{1/4} \sqrt{\log(n/\delta)} (\log(1/\gamma))^{1/8}}{\sqrt{n\epsilon}}  \Big\} \Big)$. Then for any $\gamma \in (4\max\{\exp(-d/8), \delta\},1)$, with probability at least $1-\gamma$ over the randomness in both the sample and the algorithm, there holds
\begin{equation*}
    \cR(\bw_{\priv})-\cR(\bw^{*})=    \|\bw^{*}\|_2^2\cdot    \O\bigg(  \frac{\big(d\log(1/\delta)\big)^{\frac{1}{4}} (\log(1/\gamma))^{\frac{1 }{8}}{ \sqrt{\log(n/\delta)} } }{ \sqrt{n\epsilon}  }   + \frac{\sqrt{\log(n)\log(1/\gamma){\log(n/\delta)}} }{\sqrt{n}}   \bigg).
\end{equation*}
\end{theorem}

\noindent The definition of  $\alpha$-H\"older smoothness  and the convexity of $\ell$ imply the following inequalities
\[\|\partial \ell(\bw;z)\|_2\le M+LR^\alpha\text{ and }\ell(\bw;z)\le \ell(0;z) + MR+LR^{1+\alpha},\quad\forall z\in \Z,\bw\in\W.\]
These together with Theorem~\ref{thm:dpSGD-output-Gaussian} and Theorem~\ref{thm:output-utility} imply the privacy and utility guarantees in the above theorem. The detailed proof is given in Subsection~\ref{subsec:output}.

\begin{remark}
The private SGD algorithm with output perturbation was studied in \cite{wu2017bolt} under both the Lipschitz continuity and the strong smoothness assumption, where the excess population risk for one-pass private SGD (i.e. the total iteration number $T= n$) with a bounded parameter domain was bounded by $\O\big((n\epsilon)^{-\frac{1}{2}}(d\log(1/\delta)^{\frac{1}{4}}\big)$. As a comparison, we show that the same rate (up to a logarithmic factor) $\O\big( (n\epsilon)^{-\frac{1}{2}}(d\log(1/\delta))^{\frac{1}{4}} \log^{\frac{1}{2}}(n/\delta)\big)$ can be achieved for general $\ga$-H\"{o}lder smooth losses by taking $T = \O( n^{2-\ga\over 1+\ga}+n).$ Our results extend the output perturbation for private SGD algorithms to a more general class of non-smooth losses. 
\end{remark}

\subsubsection{Differentially Private SGD with Gradient Perturbation}\label{sec:gradient}

\begin{algorithm}[t]
\begin{algorithmic}[1]
\caption{Differentially Private SGD with Gradient perturbation (\texttt{DP-SGD-Gradient})}\label{alg2}
\STATE{\bf Inputs:}  Data $S= \{z_{i}\in \Z: i=1,\ldots, n\}$, loss function $\ell(\bw,z)$ with H\"older parameters $\alpha$ and $L$, the  convex set $\W\subseteq \mathcal{B}(0, R)$, step size $\eta$, number of iterations $T$, privacy parameters $\gep$, $\gd$, and constant $\beta$.
\STATE{{\bf Set:}  $\bw_1=\0$ }
\STATE{Compute $\sigma^2=\frac{14 (M+LR^\alpha)^2 T }{\beta n^2 \epsilon} \Big( \frac{\log(1/\delta)}{(1-\beta)\epsilon} +1\Big)$ }
\FOR { $t=1$ to $T$ } 
\STATE{Sample $i_t\sim \text{Unif}([n])$} 
\STATE{$\bw_{t+1}=\proj_{\W}\big(\bw_t-\eta (\partial \ell(\bw_t;z_{i_t})+\bb_t)\big)$, where $\bb_t\sim \N(0,\sigma^2 \mathbf{I}_d)$}
\ENDFOR
\STATE  {\bf return:}   ${\bw}_\priv = \frac{1}{T}\sum_{t=1}^T\bw_t$ 
\end{algorithmic}
\end{algorithm}

An alternative approach to achieve $(\epsilon,\delta)$-DP is gradient perturbation, i.e., adding Gaussian noise to the stochastic gradient at each update. The detailed algorithm is described in Algorithm~\ref{alg2}, whose privacy guarantee is established in the following theorem.

\begin{theorem}[Privacy guarantee]\label{thm:dpSGD-gradient-Holder} 
Suppose the loss function $\ell$ is nonnegative, convex and  $\alpha$-H\"older smooth with parameter $L$.  Then Algorithm~\ref{alg2} (\texttt{DP-SGD-Gradient})   satisfies $(\gep,\gd)$-DP
if there exists $\beta \in (0,1)$ such that
$\frac{\sigma^2}{ 4(M+LR^\alpha)^2 } \ge 0.67 $
and $\lambda -1 \le \frac{\sigma^2}{6(M+LR^\alpha)^2} \log\Big(\frac{n}{\lambda(1+ \frac{\sigma^2}{4(M+LR^\alpha)^2} )}\Big)$ hold with $\lambda=\frac{\log(1/\delta)}{(1-\beta)\epsilon}+1$. 
\end{theorem}

\noindent  Since $\W\subseteq \mathcal{B}(0, R)$,  the H\"older smoothness of $\ell$ implies that $\|\partial \ell(\bw_t,z)\|_2\le M+LR^\alpha$ for any $t\in[T]$ and any $z\in \Z$, from which we know that the $\ell_2$-sensitivity of the function $\mathcal{M}_t=\partial \ell(\bw_t,z)$ can be bounded by $2(M+LR^\alpha)$. By Lemma~\ref{lem:uniform} and the post-processing property of DP, it is easy to show that the update of $\bw_t$ satisfies  $(\frac{\log(1/\delta)}{(1-\beta)\epsilon}+1, \frac{\beta \epsilon}{T})$-RDP for any $t\in[T]$. Furthermore, by the composition theorem of RDP and the relationship between $(\epsilon,\delta)$-DP and RDP, we can show that the proposed algorithm satisfies $(\epsilon,\delta)$-DP. The detailed proof can be found in Subsection~\ref{subsec:gradient}.

\bigskip
Other than the privacy guarantees, the \texttt{DP-SGD-Gradient} algorithm also enjoys utility guarantees as stated in the following theorem.

\begin{theorem}[Utility guarantee]\label{thm:gradient-utility-proj}
Suppose the loss function $\ell$ is nonnegative, convex and $\alpha$-H\"older smooth with parameter $L$.  Let $\bw_{\priv}$ be the output produced by Algorithm~\ref{alg2}
with $\eta=\frac{1}{T} \max\big\{ \frac{\sqrt{ \log(n)\log(n/\gamma) \log(1/\gamma)  }   }{\sqrt{n}}, \frac{  \sqrt{d \log(1/\delta)} (\log(1/\gamma))^{\frac{1}{4}}}{n\epsilon} \big\} $. Furthermore, let $T\asymp n^{
\frac{2-\alpha}{1+\alpha}} $ if $\alpha < \frac{1}{2}$, and $T\asymp n$ else. Then, for any $\gamma \in (18\exp(-Td/8),1)$, with probability at least $1-\gamma$ over the randomness in both the sample and the algorithm, there holds
\begin{equation*}
    \cR(\bw_{\priv})-\cR(\bw^{*})= \|\bw^{*}\|_2^{2}\cdot\O\bigg(   \frac{ \sqrt{d\log(1/\delta) \log(1/\gamma) }}{n\epsilon} + \frac{\sqrt{\log(n){\log(n/\gamma) }\log(1/\gamma) }   }{ \sqrt{n} } \bigg).
\end{equation*}
\end{theorem}

\noindent 
Our basic idea to prove Theorem \ref{thm:gradient-utility-proj} is to use the following error decomposition: 
\begin{equation}\label{eq:Gradient-error-decom2}
    \cR(\bw_{\priv})-\cR(\bw^{*})=[\cR(\bw_{\priv})-\cR_S(\bw_{\priv})]
    +[\cR_S(\bw_{\priv})-\cR_S(\bw^{*})] +[\cR_S(\bw^{*})-\cR(\bw^{*})]. 
\end{equation} 
Similar to the proof of Theorem~\ref{thm:output-utility}, the generalization error $\cR(\bw_{\priv})-\cR_S(\bw_{\priv})$ can be handled by the UAS bound, the optimization error $\cR_S(\bw_{\priv})-\cR_S(\bw^{*})$ can be estimated by standard techniques in optimization \cite[e.g.][]{Nem}, and the last term $\cR_S(\bw^{*})-\cR(\bw^{*})$ can be bounded by the Hoeffding  inequality. 
The detailed proof can be found in Subsection~\ref{subsec:gradient}.

\begin{remark}
We now compare our results with the related work under a bounded domain assumption.
The work \cite{bassily2019private} established the optimal rate $\O(\frac{1}{n\epsilon}{\sqrt{d\log(1/\delta)}}+\frac{1}{\sqrt{n}})$ for the excess population risk of private SCO algorithm in either smooth case ($\alpha=1$) or non-smooth case ($\alpha=0$). However, their algorithm has a large  gradient complexity  $\O\Big(n^{4.5}\sqrt{\epsilon}+ \frac{n^{6.5} \epsilon^{4.5}}{(d\log(\frac{1}{\delta}))^2}\Big)$. 
The work \cite{feldman2020private} proposed a private phased ERM algorithm for SCO, which can achieve the optimal excess population risk for non-smooth losses with a better gradient complexity of the order $\O(n^2\log(1/{\delta}))$. The very recent work \cite{bassily2020stability} improved the gradient complexity to $\O(n^2)$.  As a comparison, we show that SGD with gradient complexity  $\O( n^{2-\ga\over 1+\ga}+ n)$ is able to achieve the optimal (up to logarithmic terms) excess population risk $\O(\frac{1}{n\epsilon}{\sqrt{d\log(1/\delta)}}+\frac{1}{\sqrt{n}})$ for general $\alpha$-H\"older smooth losses. Our results match the existing gradient complexity for both the smooth case in \cite{bassily2019private} and the Lipschitz continuity case \cite{bassily2020stability}. An interesting observation is that our algorithm can achieve the optimal utility guarantee with the linear  gradient complexity $\O(n)$ for $\ga\ge 1/2$, which shows that a relaxation of the strong smoothness from $\alpha=1$ to $\alpha\geq1/2$ does not bring any harm in both the generalization and computation complexity.
\end{remark}

Now, we give a sufficient condition for the existence of $\beta$ in Theorem~\ref{thm:dpSGD-gradient-Holder} under a specific parameter setting.  
\begin{lemma}\label{lem:condition-beta}
    Let $n \ge 18$, $T=n$ and $\delta=1/{n^2}$. If $ \epsilon \ge   \frac{ 7(n^{\frac{1}{3}}-1) + 4\log(n) n +7 }{ 2n(n^{\frac{1}{3}}-1) }, $
    then there exists $\beta \in (0,1)$ such that  Algorithm~\ref{alg2} satisfies $(\epsilon,\delta)$-DP. 
\end{lemma}
 
\begin{figure*}[!htbp]
  \centering
  \includegraphics[scale=0.5]{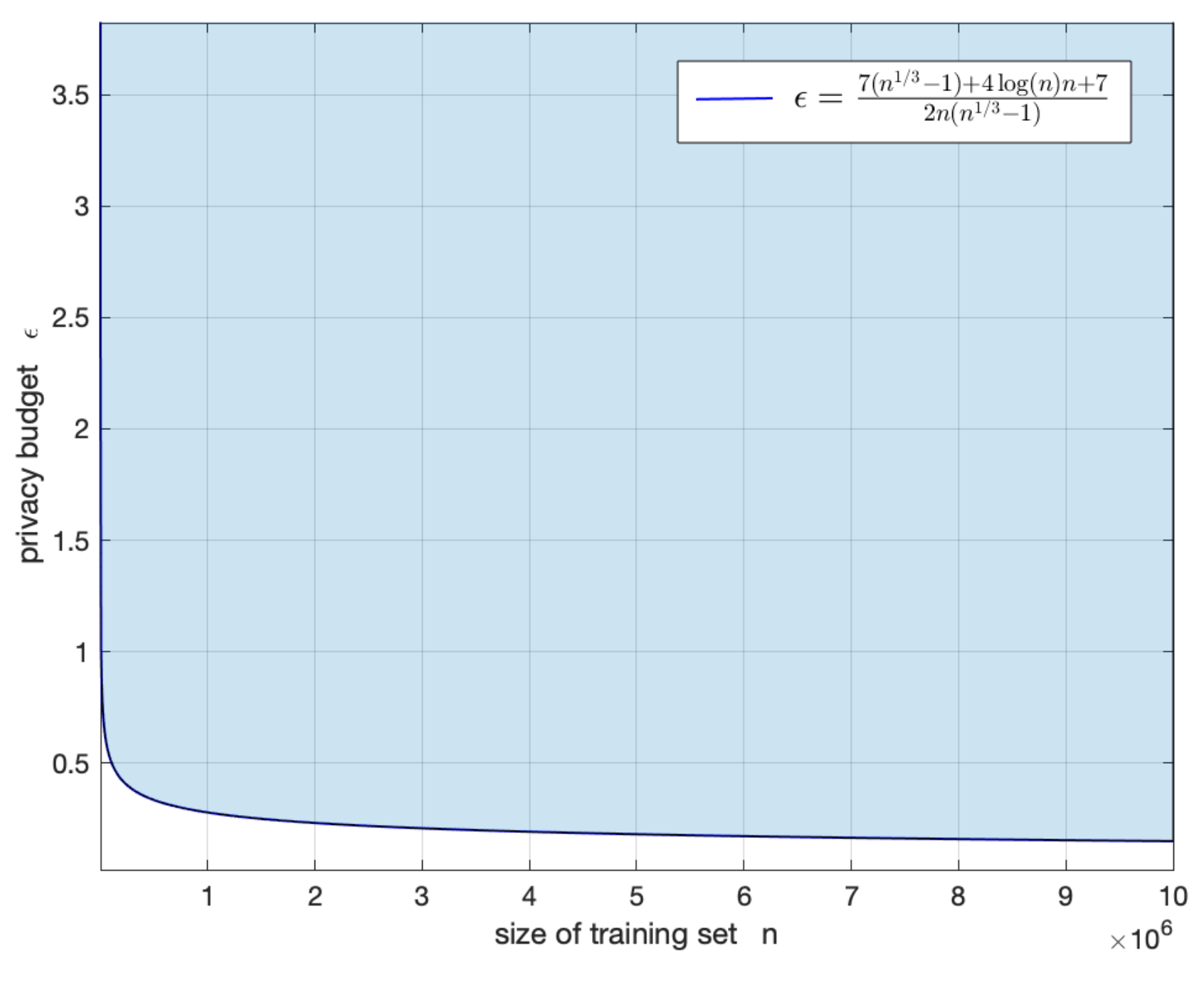} 
    \caption{The sufficient condition for the existence of $\beta$ in Lemma~13. The shaded area is the area where the sufficient condition in Lemma 13 holds true, i.e., $ \epsilon \ge   \big( 7(n^{\frac{1}{3}}-1) + 4\log(n) n +7 \big)/\big( 2n(n^{\frac{1}{3}}-1)  \big)$.  \label{fig:beta}} 
\end{figure*}

{
\begin{remark}
Privacy parameters $\epsilon$ and $\delta$ together quantify the privacy risk. 
$\epsilon$ is often called the privacy budget controlling the degree of privacy leakage. A larger value of $\epsilon$ implies higher privacy risk. Therefore, the value of $\epsilon$ depends on how much privacy the user needs to protect. Theoretically, the value of $\epsilon$ is less than 1. However, in practice,  to obtain the desired utility, a larger privacy budget, i.e., $\epsilon \ge 1$, is always acceptable  \cite{wu2017bolt,song2015learning}. For instance, Apple uses a privacy budget $\epsilon=8$ for Safari Auto-play intent detection, and $\epsilon=2$ for Health types\footnote{\url{https://www.apple.com/privacy/docs/Differential_Privacy_Overview.pdf}}. 
Parameter $\delta$ is the probability with which $e^\epsilon$ fails to bound the ratio between the two probabilities in the definition of differential privacy, i.e., the probability of privacy protection failure. For meaningful privacy guarantees, according to \cite{dwork2014algorithmic} the value of $\delta$ should be much smaller than $1/n$. In particular, we always choose $\delta=1/n^2$.  For \texttt{DP-SGD-Gradient} algorithm,  
another constant we should discuss is $\beta$ which depends on the choice of the number of iterations $T$, size of training data $n$, privacy parameters $\epsilon$ and $\delta$. 
The appearance of this parameter is due to the use of subsampling result for RDP (see Lemma~\ref{lem:uniform}).  The condition in Lemma~\ref{lem:condition-beta} ensures the existence of $\beta\in (0,1)$ such that Algorithm~\ref{alg2} satisfies DP. In practical applications, we search in $(0,1)$ for all $\beta$ that satisfy the RDP conditions in Theorem~\ref{thm:dpSGD-gradient-Holder}. Note that the closer the $\beta$ is to $1/2$, the smaller the variance of the noise added to the algorithm in each iteration. Therefore, we choose the value that is closest to $1/2$ of all $\beta$ that meets the RDP conditions as the value of $\beta$.
\end{remark}
}

We end this section with a final remark on the challenges of proving DP for Algorithm \ref{alg2} when $\W$ is   unbounded.

\begin{remark} To make Algorithm~\ref{alg2} satisfy DP when $\W= \R^d $, the variance $\sigma_t$ of the noise $\bb_t$ added in the $t$-th iteration should be proportional to the $\ell_2$-sensitivity $ \Delta_{t} =  \|  \partial \ell(\bw_{t },z_{i_{t }}) - \partial \ell(\bw_{t },z'_{i_{t }})\|_2$. 
The definition of H\"older smoothness implies that $\Delta_{t}\le 2(M+L\|\bw_{t}\|^{\alpha}_2)$. When $\alpha=0$, we have $\Delta_{t}\le 2(M+L)$ and the privacy guarantee can be established in a way similar to Theorem~\ref{thm:dpSGD-gradient-Holder}. When $\alpha\in (0,1]$, we have to establish an upper bound of $\|\bw_t\|_2$. Since $\bw_{t}=\bw_{t-1}-\eta ( \partial \ell( \bw_{t-1},z_{i_{t-1}} ) + \bb_{t-1})$ ($ \bb_{t-1} \sim \N(0, \sigma_{t-1}^2 \mathbf{I}_{d})$), we can only give a bound of $\|\bw_t\|_2$ with high probability. Thus, the  sensitivity $\Delta_t$ can not be uniformly bounded in this case. Therefore, the first challenge is how to analyze the privacy guarantee when the sensitivity changes at each iteration and all of them can not be uniformly bounded. Furthermore, by using the property of the Gaussian vector, we can prove that $\|\bw_{t }\|_2=\O( \sqrt{ t \eta} +  \eta \sum_{j=1}^{t-1} \sigma_j +  \eta \sqrt{ d \sum_{ j=1 }^{t-1}  \sigma_j^2})$ with high probability. However, as mentioned above, the variance  $\sigma_{t }$ should be proportional to $\Delta_{t }$ whose upper bound involves $\|\bw_{t }\|^\alpha_2$. Thus, $\sigma_t$ is  proportional to  $   (t\eta)^{\alpha/2} + \eta^\alpha (\sum_{j=1}^{t-1} \sigma_j)^\alpha + \eta^\alpha  (d\sum_{j=1}^{t-1} \sigma_j^2)^{\alpha/2}.$ For this reason, it seems difficult to give a clear expression for an upper bound of $\|\bw_t\|_2.$  \end{remark}

\section{Proofs of Main Results}\label{sec:proofs}
{ 
Before presenting the detailed proof, we first introduce some  useful lemmas on the concentration behavior of random variables. 
\begin{lemma}[Chernoff bound for Bernoulli variable \cite{wainwright2019high}]\label{lem:chernoff-Bernoulli}
 Let $X_1,\ldots,X_k$ be independent random variables taking values in $\{0,1\}$. Let $X=\sum_{i=1}^{k}X_i$ and $\mu=\EX[X]$. The following statements hold.
 \begin{enumerate}[label=({\alph*})]
\item For any $\Tilde{ \gamma}\in(0,1)$,  with probability at least $1-\exp\big(-\mu\Tilde{ \gamma}^2/3\big)$, there holds  $X\leq (1+\Tilde{ \gamma})\mu$.
\item For any $\Tilde{ \gamma}
\ge 1$, with probability at least $1-\exp\big(-\mu\Tilde{ \gamma}/3\big)$, there holds  $X\leq (1+\Tilde{ \gamma})\mu$.
 \end{enumerate}
\end{lemma}
\begin{lemma}[Chernoff bound for the $\ell_2$-norm of Gaussian vector \cite{wainwright2019high}] \label{lem:chernoff-gaussian}
Let $X_1,\ldots,X_k$ be i.i.d. standard Gaussian random variables, and $\mathbf{X}=[X_1,\ldots,X_k]\in \R^k$. Then for any $t\in (0,1)$, with probability at least $1-\exp(-kt^2/8)$, there holds
$ \|\mathbf{X}\|_2^2\le  k(1+t).$
\end{lemma}
\begin{lemma}[Hoeffding inequality \cite{hoeffding1994probability}]\label{lem:hoeffding}
 Let $X_1, \ldots, X_k$ be independent random variables such that $a_i \le X_i \le b_i$ with probability 1 for all $i\in[k]$. Let $ {X}=\frac{1}{k}\sum_{i=1}^k X_i$.  Then for any $t>0$, with probability at least $1-\exp(-2t^2/\sum_i (b_i-a_i)^2)$, there holds $  {X}-\EX[ {X}] \le t.  $
\end{lemma}
\begin{lemma}[Azuma-Hoeffding inequality \cite{hoeffding1994probability}]\label{lem:azuma}
 Let $X_1, \ldots, X_k$ be a sequence of random variables   where  $X_i$ may depend on the previous random variables $X_1,\ldots,X_{i-1}$ for all $i=1,\ldots,k$. Consider a sequence of functionals $\xi_i(X_1,\ldots,X_i)$, $i\in[k]$. If $| \xi_i -\EX_{X_i}[\xi_i] | \le b_i$ for each $i$. Then for all $t> 0$, with probability at least $1-\exp(-t^2/(2\sum_i b_i^2))$, there holds $ \sum_{i=1}^k \xi_i - \sum_{i=1}^k \EX_{X_i}[\xi_i] \le t $.  
\end{lemma}
\begin{lemma}[Tail bound of sub-Gaussian variable \cite{wainwright2019high}]\label{lem:tailbound-subG}
 Let $X$ be a sub-Gaussian random variable with mean $\mu$ and sub-Gaussian parameter $v^2$. Then,  for any $t\ge 0$, we have, with probability at least $1-\exp\big(-t^2/(2v^2)\big)$, that
 $X-\mu \le t$. 
\end{lemma}
}

\subsection{Proofs on UAS bound of SGD on Non-smooth Losses} \label{subsec:UAS}
Our stability analysis for unbounded domain requires the following lemma on the self-bounding property for H{\"o}lder smooth losses.
\begin{lemma}(\cite{lei2020fine,ying2017unregularized})\label{lem:self-bounding}
Suppose the loss function $\ell$ is  nonnegative, convex and $\alpha$-H\"older smooth with parameter $L$.   Then for $c_{\ga,1}$ defined as \eqref{alpha-1} we have
  \[
  \|\partial \ell(\bw,z)\|_2\leq c_{\ga,1}\ell^{\frac{\ga}{1+\ga}}(\bw,z),\quad\forall \bw\in\R^d,z\in\Z.
  \]
\end{lemma}

Based on Lemma~\ref{lem:self-bounding}, we develop the following bound on the iterates produced by the SGD update \eqref{eq:pro-SGD-updaterule} which is critical to analyze the privacy and utility guarantees in the case of unbounded domain. Recall that $M=\sup_{z\in \Z}\|\partial \ell(0,z)\|_2$.

\begin{lemma}\label{lem:boundedness}
Suppose the loss function $\ell$ is   nonnegative, convex and $\alpha$-H\"older smooth with parameter $L$.  Let $\{\bw_t\}_{t=1}^{T}$ be the sequence produced by SGD with $T$ iterations when $\W=\R^d$ and $\eta_t < \min\{1,  1/L\}$.  Then, for any $t\in[T]$, there holds
$$\|\bw_{t+1}\|_2^2\le  C_\ga \sum_{j=1}^t \eta_j, $$
where $C_\ga=\frac{1-\ga}{1+\ga}  c^{\frac{2(1+\ga)}{1-\ga}}_{\ga,1} \big( \frac{\ga}{1+\ga}\big)^{\frac{2\ga}{1-\ga}}+ 2\sup_{z\in \Z} \ell(0;z) $.
\end{lemma}
\begin{proof}
The update rule $\bw_{t+1}=\bw_t-\eta_t\partial \ell(\bw_t,z_{i_t})$ implies that 
\begin{align}\label{eq:SGDupdate-1}
    \|\bw_{t+1}\|_2^2&=\|\bw_t-\eta_t\partial \ell(\bw_t,z_{i_t})\|_2^2 = \|\bw_t\|_2^2 + \eta_t^2\|\partial \ell(\bw_t,z_{i_t})\|_2^2 - 2\eta_t \langle \bw_t, \partial \ell(\bw_t,z_{i_t})\rangle.
\end{align}
First, we consider the case $\ga =0$. By the definition of H\"older smoothness, we know $\ell$ is $(M+L)$-Lipschitz continuous. Furthermore, by the convexity of $\ell$, we have
\begin{align*}
    \eta_t \|\partial \ell(\bw_t,z_{i_t})\|_2^2 -  2\langle \bw_t, \partial \ell(\bw_t,z_{i_t}) &\le \eta_t \|\partial \ell(\bw_t,z_{i_t})\|_2^2 + 2\big( \ell(0,z_{i_t}) - \ell(\bw_t,z_{i_t}) \big) \nonumber\\
    &\le (M+L)^2 + 2\sup_{z\in \Z} \ell(0,z),
\end{align*}
where in the last inequality we have used $\eta_t < 1$ and the nonnegativity of $\ell$. Now, putting the above inequality back into \eqref{eq:SGDupdate-1} and 
taking the summation gives
\begin{equation}\label{eq:SGD-boundness-1}
\|\bw_{t+1}\|_2^2 \le \big( (M+L)^2 + 2\sup_{z\in \Z} \ell(0;z) \big) \sum_{j=1}^t \eta_j.
\end{equation}

Then, we consider the case $\ga=1$. In this case,  Lemma~\ref{lem:self-bounding} implies $\|\partial \ell(\bw;z)\|_2^2\le 2L \ell(\bw;z)$.
Therefore,
\begin{align*}
    &\eta_t \|\partial \ell(\bw_t,z_{i_t})\|_2^2 - 2\langle \bw_t, \partial \ell(\bw_t,z_{i_t})\rangle \le  2\eta_t L \ell(\bw_t,z_{i_t}) + 2\ell(0,z_{i_t})-2\ell(\bw_t,z_{i_t})
    \le 2\ell(0,z_{i_t}),
\end{align*}
where we have used the convexity of $\ell$ and $\eta_t < 1/L$. Plugging the above inequality back into \eqref{eq:SGDupdate-1} and taking the summation yield that 
\begin{equation}\label{eq:SGD-boundness-2}
\|\bw_{t+1}\|_2^2 \le 2 \sup_{z\in \Z} \ell(0,z) \sum_{j=1}^t \eta_j.
\end{equation}

Finally, we consider the case $\ga \in (0,1)$.
According to the self-bounding property and the convexity, we know
$$ \|\partial \ell(\bw_t,z_{i_t})\|_2 \le c_{\ga,1}  \ell^{\frac{\ga}{1+\ga}}(\bw_t,z_{i_t})\le c_{\ga,1} \big(\langle \bw_t, \partial \ell(\bw_t,z_{i_t}) \rangle + \ell(0,z_{i_t})\big)^{\frac{\ga}{1+\ga}}.$$
Therefore, for $\ga \in (0,1)$ there holds
\begin{align*}
        \|\partial \ell(\bw_t,z_{i_t})\|_2^2 &\le c^2_{\ga,1} \big(\langle \bw_t, \partial \ell(\bw_t,z_{i_t}) \rangle + \ell(0,z_{i_t})\big)^{\frac{2\ga}{1+\ga}} \nonumber\\
    &= \Big(\frac{1+\ga}{\ga \eta_t} \big(\langle \bw_t, \partial \ell(\bw_t,z_{i_t}) \rangle + \ell(0,z_{i_t})\big) \Big)^{\frac{2\ga}{1+\ga}} \cdot \Big( c^2_{\ga,1} \big(\frac{1+\ga}{\ga \eta_t}\big)^{-\frac{2\ga}{1+\ga}}\Big)\nonumber\\
    &\le \frac{2\ga}{1+\ga} \Big( \frac{1+\ga}{\ga \eta_t} \big(\langle \bw_t, \partial \ell(\bw_t,z_{i_t}) \rangle + \ell(0,z_{i_t})\big) \Big) + \frac{1-\ga}{1+\ga} \Big( c^2_{\ga,1} \big(\frac{1+\ga}{\ga \eta_t}\big)^{-\frac{2\ga}{1+\ga}} \Big)^{\frac{1+\ga}{1-\ga}} \nonumber\\
    &= 2\eta_t^{-1} \big(\langle \bw_t, \partial \ell(\bw_t,z_{i_t}) \rangle + \ell(0,z_{i_t})\big) + \frac{1-\ga}{1+\ga}  c^{\frac{2(1+\ga)}{1-\ga}}_{\ga,1} \big( \frac{\ga}{1+\ga}\big)^{\frac{2\ga}{1-\ga}}  \eta_t^{\frac{2\ga}{1-\ga}},
\end{align*}
where the last inequality used Young's inequality $ab\le \frac{1}{p}a^p + \frac{1}{q}b^q$ with $\frac{1}{p}+\frac{1}{q}=1.$
Putting the above inequality into \eqref{eq:SGDupdate-1}, we have
$$\|\bw_{t+1}\|_2^2 \le  \|\bw_t\|_2^2  + \frac{1-\ga}{1+\ga}  c^{\frac{2(1+\ga)}{1-\ga}}_{\ga,1} \big( \frac{\ga}{1+\ga}\big)^{\frac{2\ga}{1-\ga}} \eta_t^{\frac{2}{1-\ga}} + 2\ell(0,z_{i_t})\eta_t, $$
If the step size $\eta_t < 1$, then
$$\|\bw_{t+1}\|_2^2  \le  \|\bw_t\|_2^2 + \bigg(\frac{1-\ga}{1+\ga}  c^{\frac{2(1+\ga)}{1-\ga}}_{\ga,1} \big( \frac{\ga}{1+\ga}\big)^{\frac{2\ga}{1-\ga}} + 2\sup_{z\in \Z} \ell(0;z) \bigg) \eta_t. $$
Taking a summation of the above inequality, we get
\begin{equation}\label{eq:SGD-boundness-3}
\|\bw_{t+1}\|_2^2 \le \Big(\frac{1-\ga}{1+\ga}  c^{\frac{2(1+\ga)}{1-\ga}}_{\ga,1} \big( \frac{\ga}{1+\ga}\big)^{\frac{2\ga}{1-\ga}}+ 2\sup_{z\in \Z} \ell(0;z) \Big) \sum_{j=1}^t \eta_j .
\end{equation}
The desired result follows directly from  \eqref{eq:SGD-boundness-1}, \eqref{eq:SGD-boundness-2} and \eqref{eq:SGD-boundness-3} for different values of $\alpha.$
\end{proof}

The following lemma shows the approximately non-expensive behavior of the gradient mapping $\bw\mapsto \bw - \eta \partial \ell(\bw, z)$. 
The case $\ga \in [0,1)$ can be found  in Lei and Ying~\cite{lei2020fine}, and the case $\ga=1$ can be found  in Hardt~\cite{hardt2016train}.
\begin{lemma}\label{lem:holder-expansive}
Suppose the loss function $\ell$ is convex  and $\alpha$-H\"older smooth with parameter $L$. Then for all $\bw, \bw'\in\R^d$ and {$\eta \le 2/L$} there holds
  \[
  \|\bw-\eta \partial \ell(\bw,z)-\bw'+\eta \partial \ell(\bw',z)\|_2^2 \leq \|\bw-\bw'\|_2^2+\frac{1-\ga}{1+\ga}(2^{-\ga}L)^\frac{2}{1-\ga}\eta^{\frac{2}{1-\ga}}.
  \]
\end{lemma}

With the above preparation,  we are now ready to prove Theorem \ref{thm:difference-w}.
\begin{proof}[Proof of Theorem~\ref{thm:difference-w}]
\noindent (a)
Assume that $S$ and $S'$ differ by the $i$-th datum, i.e., $z_{i} \ne z'_{i}.$ Let $\{\bw_t\}_{t=1}^T$ and $\{\bw'_t\}_{t=1}^T$ be the sequence produced by SGD update \eqref{eq:pro-SGD-updaterule} based on $S$ and $S'$, respectively.
For simplicity, let $ c^2_{\ga,2}=\frac{1-\ga}{1+\ga}(2^{-\ga}L)^\frac{2}{1-\ga}$. Note that when $\W=\R^d$, Eq. \eqref{eq:pro-SGD-updaterule} reduces to $\bw_{t+1}=\bw_t-\eta \partial \ell(\bw_t,z_{i_t}) $. For any $t\in[T]$, we consider the following two cases.

\noindent \textit{Case 1:} If $i_t\ne i$, Lemma~\ref{lem:holder-expansive} implies that
$$\|\bw_{t+1}-\bw'_{t+1}\|_2^2 = \|\bw_{t}-\eta_t\partial \ell(\bw_t,z_{i_t})-\bw'_{t}+\eta_t\partial \ell(\bw'_t,z_{i_t})\|_2^2 \le \|\bw_{t}-\bw'_{t}\|_2^2 +  c^2_{\ga,2}\eta_t^{\frac{2}{1-\ga}}.$$

\noindent \textit{Case 2:} If $i_t= i$, it follows from the elementary inequality $(a+b)^2 \le (1+p)a^2+(1+1/p)b^2$ that
\begin{align*}
    \|\bw_{t+1}-\bw'_{t+1}\|_2^2
    &= \|\bw_t-\eta_t \partial \ell(\bw_t,z_i)-\bw'_t + \eta_t \partial \ell(\bw'_t,z'_i)\|_2^2 \\
    &\le (1+p) \|\bw_{t}-\bw'_{t}\|_2^2 + (1+1/p) \eta_t^2 \|\partial \ell(\bw'_t,z'_i) -  \partial \ell(\bw_t,z_i)\|_2^2.
\end{align*}
According to the definition of H\"older smoothness  and Lemma~\ref{lem:boundedness}, we know
\begin{equation}\label{eq:subgradient-boundness}\|\partial \ell(\bw_{t},z)\|_2\le M + L \Big(C_\ga \sum_{j=1}^{t-1}\eta_j\Big)^{\frac{\ga}{2}} := c_{\ga,t}. \end{equation}
Combining the above two cases and \eqref{eq:subgradient-boundness} together, we have
\[
\|\bw_{t+1}-\bw'_{t+1}\|_2^2 \le (1+p)^{\mbI_{[i_t =  i ]}} \|\bw_{t}-\bw'_{t}\|_2^2 +  c^2_{\ga,2}\eta_t^{\frac{2}{1-\ga}} +4(1+1/p)\mbI_{[i_t =  i ]}c^2_{\ga,t} \eta_t^2, \]
where $\mbI_{[i_t =  i ]}$ is the indicator function, i.e., $\mbI_{[i_t =  i ]}=1$ if $i_t=i$ and $0$ otherwise. Applying the above inequality recursively, we get
\[
\|\bw_{t+1}-\bw'_{t+1}\|_2^2\leq \prod_{k=1}^t (1+p)^{\mbI_{[i_{k} =  i ]}} \|\bw_{1}-\bw'_{1}\|_2^2 + \Big(c^2_{\ga,2} \sum_{k=1}^t \eta_{k}^{\frac{2}{1-\ga}}
    + 4 \sum_{k=1}^t c_{\ga,k}^2 \eta_{k}^2  (1+1/p){\mbI_{[i_{k}=i]}}\Big)\prod_{j=k+1}^t  (1+p)^{\mbI_{[i_{j} =  i ]}}.
\]
Since $\bw_1=\bw_1'$ and  $\eta_t=\eta$, we further get
\begin{align}\label{eq:difference-w}
  \|\bw_{t+1}-\bw'_{t+1}\|_2^2 & \leq \prod_{j=2}^{t}(1+p)^{\mbI_{[i_j=i]}} \Big(c^2_{\alpha,2}t \eta^{\frac{2}{1-\alpha}}+4\eta^2\sum_{k=1}^{t}c^2_{\alpha,k}(1+1/p){\mbI_{[i_k=i]}}\Big) \notag\\
  & \leq  (1+p)^{\sum_{j=2}^{t}\mbI_{[i_j=i]}}\Big(c^2_{\alpha,2}t\eta^{\frac{2}{1-\alpha}}+4c^2_{\alpha,t}\eta^2(1+1/p)\sum_{k=1}^{t}\mbI_{[i_k=i]}\Big).
\end{align}
Applying Lemma~\ref{lem:chernoff-Bernoulli} with  $X_j=\mbI_{[i_{j} =  i ]}$ and  $X=\sum_{j=1}^t X_j$, for any $\exp(-t/3n)\le \gamma \le 1$, 
with probability at least $1-\frac{\gamma}{n}$, there holds
\[
\sum_{j=1}^{t}\mbI_{[i_j=i]}\leq \frac{t}{n}\Big(1+\frac{\sqrt{3\log(1/\gamma)}}{\sqrt{t/n}}\Big).
\]
For any $0<\gamma <\exp(-t/3n)$, with probability at least $1-\frac{\gamma}{n}$, there holds
\[
\sum_{j=1}^{t}\mbI_{[i_j=i]}\leq \frac{t}{n}\Big(1+\frac{3\log(1/\gamma)}{t/n}\Big).
\] 
Plug the above two inequalities back into \eqref{eq:difference-w}, and let $c_{\gamma,t}=\max\Big\{\sqrt{\frac{3\log(n/{\gamma})}{t/n}}, \frac{3\log(n/{\gamma})}{t/n} \Big\}$. Then, for any $\gamma \in (0,1)$,   with probability at least $1-\frac{\gamma}{n}$, we have 
\[
  \|\bw_{t+1}-\bw'_{t+1}\|_2^2 \leq (1+p)^{\frac{t}{n}(1+c_{\gamma,t})}\Big(c^2_{\alpha,2}t\eta^{\frac{2}{1-\alpha}}
  +4c^2_{\alpha,t}\eta^2(1+1/p)\frac{t}{n}(1+c_{\gamma,t})\Big).
\]
Let $p=\frac{1}{\frac{t}{n}(1+c_{\gamma,t})}$. Then we know
$
(1+p)^{\frac{t}{n} (1+c_{\gamma,t}) }\leq e 
$
and therefore
\begin{align}\label{eq:difference-w-2}
     \|\bw_{t+1}-\bw'_{t+1}\|_2^2 \leq e\Big(c^2_{\alpha,2}t\eta^{\frac{2}{1-\alpha}}+4c^2_{\alpha,t}\eta^2\Big(1+\frac{t}{n}(1+c_{\gamma,t})\Big)\frac{t}{n}(1+c_{\gamma,t})\Big).
\end{align}
This together with the inequality
$c^2_{\ga,t}\leq\big(M+L (C_\ga t \eta)^{\frac{\ga}{2}}\big)^2$ due to Lemma \ref{lem:boundedness}, we have, with probability at least $1-\frac{\gamma}{n}$, that 
\begin{equation*}
  \|\bw_{t+1}-\bw'_{t+1}\|_2^2 \leq e\Big(c^2_{\alpha,2}t\eta^{\frac{2}{1-\alpha}}+4\big(M+L (C_\ga t \eta)^{\frac{\ga}{2}}\big)^2\eta^2\Big(1+\frac{t}{n}(1+c_{\gamma,t})\Big)\frac{t}{n}(1+c_{\gamma,t})\Big). 
\end{equation*} By taking a union bound of probabilities over $i=1,\ldots,n$, with probability at least $1-\gamma$, there holds
\begin{equation*}
  \sup_{S\simeq S'} \|\bw_{t+1}-\bw'_{t+1}\|_2^2 \leq e\Big(c^2_{\alpha,2}t\eta^{\frac{2}{1-\alpha}}+4\big(M+L (C_\ga t \eta)^{\frac{\ga}{2}}\big)^2\eta^2\Big(1+\frac{t}{n}(1+c_{\gamma,t})\Big)\frac{t}{n}(1+c_{\gamma,t})\Big).
\end{equation*}
Let $\Delta_{SGD}(\gamma)=\Big(e\big(c^2_{\alpha,2}T\eta^{\frac{2}{1-\alpha}}+4\big(M+L (C_\ga T \eta)^{\frac{\ga}{2}}\big)^2\eta^2\Big(1+\frac{T}{n}(1+c_{\gamma,T})\Big)\frac{T}{n}(1+c_{\gamma,T})\big)\Big)^{1/2}$.  Recall that $\A$ is the SGD with $T$ iterations, and $\bar{\bw}=\frac{1}{T}\sum_{t=1}^T \bw_t$ is the output produced by $\A$. Hence, $ \sup_{S\simeq S'} \delta_{\A}(S,S')= \sup_{S\simeq S'} \|\bar{\bw}-\bar{\bw}'\|_2$. By the convexity of the $\ell_2$-norm, with probability at least $1-\gamma$, we have
$$ \sup_{S\simeq S'} \delta_{\A}(S,S')\le \frac{1}{T}\sum_{t=1}^T\sup_{S\simeq S'}\|\bw_t-\bw_t'\|_2 \le \Delta_{SGD}(\gamma).$$
This completes the proof of part (a).

\bigskip

\noindent (b) For the case $\W\subseteq \mathcal{B}(0,R)$, the analysis is similar to the case $\W= \R^d$ except using a different estimate for the term  $\|\partial \ell(\bw_t,z)\|_2$.
Indeed, in this case we have $\|\bw_t\|_2\le R$, which together with the H\"older smoothness, implies  $\|\partial \ell(\bw_t,z)\|_2\le M+LR^\alpha$ for any $t\in[T]$ and $z\in \Z$. Now, replacing  $c_{\alpha,t}=M+LR^\alpha$ in \eqref{eq:subgradient-boundness} and putting $c_{\alpha,t}$ back into \eqref{eq:difference-w-2}, with probability at least $1-\frac{\gamma}{n}$, we obtain 
\begin{equation*}
  \sup_{S\simeq S'} \| \bw_{t+1}- \bw'_{t+1}\|_2^2 \leq e\Big(c^2_{\alpha,2}t\eta^{\frac{2}{1-\alpha}}+4\big(M+LR^\alpha \big)^2\eta^2\Big(1+\frac{t}{n}(1+c_{\gamma,t})\Big)\frac{t}{n}(1+c_{\gamma,t})\Big).
\end{equation*}
Now, let $\tilde{\Delta}_{SGD}(\gamma)=\Big(e\big(c^2_{\alpha,2}T\eta^{\frac{2}{1-\alpha}}+4\big(M+LR^\alpha \big)^2\eta^2\Big(1+\frac{T}{n}(1+c_{\gamma,T})\Big)\frac{T}{n}(1+c_{\gamma,T})\big)\Big)^{1/2}.$ The convexity of a norm implies, with probability at least $1-\gamma$, that
$$ \sup_{S\simeq S'} \delta_{\A}(S,S')\le \frac{1}{T}\sum_{t=1}^T\sup_{S\simeq S'}\|\bw_t-\bw_t'\|_2  \le \tilde{\Delta}_{SGD}(\gamma).$$
The proof of the theorem is completed.
\end{proof}


\bigskip

\subsection{Proofs on Differentially Private SGD with Output Perturbation}\label{subsec:output}

In this subsection, we prove the privacy and utility guarantees for output perturbation (i.e. Algorithm~\ref{alg1}). We consider both the unbounded domain $\W=\R^d$   and bounded domain $\W \subseteq \mathcal{B}(0, R)$. 

\medskip 
We first prove  Theorem~\ref{thm:dpSGD-output-Gaussian} on the privacy guarantee of Algorithm~\ref{alg1}.  
\begin{proof}[Proof of Theorem~\ref{thm:dpSGD-output-Gaussian}]

Let $\A$ be the SGD with  $T$  iterations, $\bar{\bw}=\frac{1}{T}\sum_{t=1}^T \bw_t$ be the output of $\A$.   
First, consider the unbounded domain case, i.e., $\W=\R^d$.  Let $ I = \{i_1,\ldots, i_{T} \}$ be the sequence of sampling after $T$ iterations in $\A$. 
Define $$\mathcal{B}=\big\{I: \sup_{S\simeq S'}\delta_{\A}(S,S')\le \Delta_\sgd(\delta/2) \big\}.$$
Part (a) in Theorem \ref{thm:difference-w} implies that $\mathbb{P}(I\in\mathcal{B})\ge 1-\delta/2$. Further, according to the definitions, we know the $\ell_2$-sensitivity of $\A$ is identical to the UAS of $\A$. Thus,  
if $ I\in \mathcal{B} $, then Lemma \ref{lem:gaussian-noise} with $\delta'=\delta/2$ implies Algorithm \ref{alg1} satisfies $(\epsilon,\delta/2)$-DP. For any neighboring datasets $S$ and $S'$, let $\bw_{\priv}$ and $\bw'_{\priv}$ be the output produced by Algorithm \ref{alg1} based on $S$ and $S'$, respectively. Hence, for any  $E\subseteq \R^d$ we have
\begin{align*}
    \mathbb{P}(\bw_\priv \in E)&= \mathbb{P}(\bw_\priv \in E \cap I\in \mathcal{B})+\mathbb{P}(\bw_\priv \in E \cap I\in \mathcal{B}^c) \\
    & \le \mathbb{P}(\bw_\priv \in E | I\in \mathcal{B})\mathbb{P} (I\in\mathcal{B}) + \frac{\delta}{2} 
    \le \Big(e^\epsilon \mathbb{P}(\bw_\priv'\in E|I\in \mathcal{B}) + \frac{\delta}{2}\Big)\mathbb{P}(I\in \mathcal{B}) +\frac{\delta}{2}\\
    &\le e^\epsilon \mathbb{P}(\bw_\priv'\in E\cap I\in \mathcal{B}) + \delta
    \le e^\epsilon \mathbb{P}(\bw_\priv'\in E ) + \delta, 
\end{align*} 
where in the second inequality we have used the definition of DP. Therefore, Algorithm~\ref{alg1} satisfies $(\epsilon,\delta)$-DP when $\W=\R^d$. 
The bounded domain case can be proved in a similar way by using part (b) of Theorem \ref{thm:difference-w}. 	The proof is completed. 
\end{proof}


\bigskip

Now, we turn to the utility guarantees of  Algorithm~\ref{alg1}. 
Recall that the excess population risk $\cR(\bw_{\priv})-\cR(\bw^{*})$ can be decomposed as follows ($\Bar{\bw}=\frac{1}{T}\sum_{t=1}^T\bw_t$)
\begin{align}\label{eq:error-decom1}
    \cR(\bw_{\priv})-\cR(\bw^{*})&=[\cR(\bw_{\priv})-\cR(\Bar{\bw})]+[\cR(\Bar{\bw})-\cR_S(\Bar{\bw})]
    +[\cR_S(\Bar{\bw})-\cR_S(\bw^{*})] +[\cR_S(\bw^{*})-\cR(\bw^{*})].
\end{align}

We now introduce three lemmas to control the  first three terms on the right hand side of  \eqref{eq:error-decom1}. The following lemma controls the error resulting from the added noise.

\begin{lemma}\label{lem:output-utility-1}
  Suppose the loss function $\ell$ is  nonnegative, convex and $\alpha$-H\"older smooth with parameter $L$. Let $\bw_{\priv}$ be the output produced by Algorithm~\ref{alg1} based on the dataset $S=\{z_1,\cdots,z_n \}$  
  with $\eta_t=\eta< \min \{1,1/L\}$.  Then for any $\gamma \in (4\exp(-d/8), 1)$,  the following statements hold true.  
  \begin{enumerate}[label=({\alph*})]
\item If $\W=\R^d$, then, with probability at least $1-\frac{\gamma}{4}$, there holds
   \begin{align*}
    \cR(\bw_{\priv})-\cR(\Bar{\bw})=\O\Big( ( T\eta)^{\frac{\alpha}{2}}\sigma\sqrt{d}(\log(1/\gamma))^{\frac{1}{4}} +\sigma^{1+\alpha}d^{\frac{1+\alpha}{2}}(\log(1/\gamma))^{\frac{1+\alpha}{4}}\Big). 
\end{align*}
\item If $\W \subseteq \mathcal{B}(0,R)$ with $R> 0$, then, with probability at least $1-\frac{\gamma}{4}$, we have
   \begin{align*}
    \cR(\bw_{\priv})-\cR(\Bar{\bw})=\O\Big(  \sigma\sqrt{d}(\log(1/\gamma))^{\frac{1}{4}} +\sigma^{1+\alpha}d^{\frac{1+\alpha}{2}}(\log(1/\gamma))^{\frac{1+\alpha}{4}}\Big). 
\end{align*}
\end{enumerate}
\end{lemma}
\begin{proof}
\noindent (a) 
First, we consider the case $\W=\R^d$. 
Note that
\begin{align}\label{eq:decom-difference-1}
    \cR(\bw_{\priv})-\cR(\Bar{\bw})&=\EX_z[\ell(\bw_{\priv},z)-\ell(\Bar{\bw},z)]\le \EX_z[\langle \partial \ell(\bw_{\priv},z), \bw_{\priv}-\Bar{\bw} \rangle]\nonumber\\
    &\le \EX_z[\|\partial \ell(\bw_{\priv},z)\|_2\|\bb\|_2]\le (M+L\|\bw_{\priv}\|_2^\alpha)\|\bb\|_2\nonumber\\
    &\le (M+L\|\Bar{\bw}\|^\alpha_2)\|\bb\|_2+L\|\bb\|_2^{1+\alpha},
\end{align}
where the first inequality is due to the convexity of $\ell$, the second inequality follows from the Cauchy-Schwartz inequality, the third inequality is due to the definition of H\"older smoothness, and the last inequality uses $\bw_{\priv}=\Bar{\bw}+\bb$.  Hence, to estimate $\cR(\bw_{\priv})-\cR(\Bar{\bw})$, it suffices to bound $\|\bb\|_2$ and $\|\bar{\bw}\|_2$.  
Since $\bb\sim \N(0,\sigma^2 \mathbf{I})$, then for any $\gamma \in(\exp(-d/8),1)$, Lemma~\ref{lem:chernoff-gaussian} implies,  with probability at least $1-\frac{\gamma}{4}$, that
\begin{equation}\label{eq:b-norm-bound}
    \|\bb\|_2\le \sigma \sqrt{d} \Big(1 + \Big( \frac{8}{d}\log\big(4/\gamma\big) \Big)^{\frac{1}{4}}\Big).
\end{equation}
Further, by the convexity of a norm and Lemma \ref{lem:boundedness}, we know
\begin{equation}\label{eq:wbar-unbounded}
\|\Bar{\bw}\|_2\le \frac{1}{T}\sum_{t=1}^T \|\bw_t\|_2\le \big(C_\alpha T\eta\big)^{\frac{1}{2}}.
\end{equation}
Putting the above inequality and \eqref{eq:b-norm-bound} back into \eqref{eq:decom-difference-1} yields
\begin{align*}
    \cR(\bw_{\priv})-\cR(\Bar{\bw})
    &\le \big(M+L(C_\alpha T\eta)^{\frac{\alpha}{2}}\big) \sigma \sqrt{d} 
    \Big(1 + \Big( \frac{8}{d}\log\big(4/\gamma\big) \Big)^{\frac{1}{4}}\Big)+L\sigma^{1+\alpha}d^{\frac{1+\alpha}{2}}\Big(1 + \Big( \frac{8}{d}\log\big(4/\gamma\big) \Big)^{\frac{1}{4}}\Big)^{1+\alpha  }\nonumber\\
    &=\O\Big( (T\eta)^{\frac{\alpha}{2}} \sigma \sqrt{d}
    \big( \log(1/\gamma) \big)^{\frac{1}{4}} + \sigma^{1+\alpha}d^\frac{1+\alpha}{2} 
    \big( \log(1/\gamma) \big)^{\frac{1+\alpha }{4}} \Big).
\end{align*}
This completes the proof of part (a).

\ 

\noindent (b) The proof for the unbounded domain case is similar to that of the bounded domain.  
Since $\|\bw_t\|_2\le R$ for $t\in [T]$ in this case, then 
\begin{equation}\label{eq:wbar-bounded}
    \|\Bar{\bw}\|_2\le \frac{1}{T}\sum_{t=1}^T \|\bw_t\|_2\le R.
\end{equation}
Plugging \eqref{eq:wbar-bounded} and \eqref{eq:b-norm-bound} back into \eqref{eq:decom-difference-1} yield the result in  part (b). 
\end{proof}

In the following lemma, we use the stability of SGD to control the generalization error $\cR(\Bar{\bw})-\cR_S(\Bar{\bw})$.
\begin{lemma}\label{lem:output-utility-2}
  Suppose the loss function $\ell$ is  nonnegative, convex, and $\alpha$-H\"older smooth with parameter $L$. Let $\A$ be the SGD 
  with $T$ iterations and  $\eta_t=\eta<\min \{1,1/L\}$ based on the dataset $S=\{z_1,\cdots,z_n\}$, and   $\Bar{\bw}=\frac{1}{T}\sum_{t=1}^T\bw_t$ be the output produced by $\A$.  Then for any $\gamma \in (4\delta, 1)$, the following statements hold true.  
  \begin{enumerate}[label=({\alph*})]
\item If $\W=\R^d$, then,  
  with probability at least $1-\frac{\gamma}{4}$, there holds
  \begin{align*}
    \cR(\Bar{\bw})-\cR_S(\Bar{\bw})
    =\O\Big((T\eta)^{\frac{\alpha}{2}} \Delta_\sgd(\delta/2) \log(n)\log(1/\gamma)+ (T\eta)^{\frac{1+\alpha}{2}}  \sqrt{n^{-\frac{1}{2}}\log(1/\gamma)}\Big).
\end{align*}
\item If $\W \subseteq \mathcal{B}(0,R)$ with $R> 0$, then, with probability at least $1-\frac{\gamma}{4}$, we have
  \begin{align*}
    \cR(\Bar{\bw})-\cR_S(\Bar{\bw})
    =\O\Big( \tilde{\Delta}_\sgd(\delta/2) \log(n)\log(1/\gamma)+   \sqrt{n^{-\frac{1}{2}}\log(1/\gamma)}\Big).
\end{align*}
\end{enumerate}
\end{lemma}
\begin{proof}
 \noindent (a) Consider the unbounded domain case. 
Part (a) in Theorem \ref{thm:difference-w}  implies, with probability at least $1-\frac{\delta}{2}$, that
\begin{equation}\label{eq:stability-w-bar}
    \sup_{S\simeq S'}\delta_{\A}(S,S') \le  \Delta_\sgd(\delta/2).  
\end{equation}
Since  $\gamma \ge 4\delta$, then we know \eqref{eq:stability-w-bar} holds with probability at least $1-\frac{\gamma}{8}$. 
According to the result $\|\Bar{\bw}\|_2\le \sqrt{C_{\alpha}T\eta}$ by \eqref{eq:wbar-unbounded} and  Lemma~\ref{lem:generror-high-probability} with $G=\sqrt{C_{\alpha}T\eta}$ together, we derive the following inequality with probability at least $1-\frac{\gamma}{8}-\frac{\gamma}{8}=1-\frac{\gamma}{4}$
\begin{align*}
    \cR(\Bar{\bw})-\cR_S(\Bar{\bw})
    &\le c \bigg( (M+L(C_{\alpha}T\eta)^{\frac{\alpha}{2}}) \Delta_\sgd(\delta/2) \log(n)\log({8}/{\gamma}) +  \big( \sup_{z\in \Z} \ell(0,z)+ (M+L(T\eta)^{\frac{\alpha}{2}})\sqrt{T\eta}\big) \sqrt{\frac{\log({8}/{\gamma})}{n}}  \bigg) \\
    &=\O\bigg((T\eta)^{\frac{\alpha}{2}} \Delta_\sgd(\delta/2) \log(n)\log(1/\gamma)+ (T\eta)^{\frac{1+\alpha}{2}}  \sqrt{\frac{ \log(1/\gamma)}{n}} \bigg),
\end{align*}
where $c>0$ is a constant. The proof of part (a) is completed. 

\ 

\noindent (b) For the case $ \W \subseteq \mathcal{B}(0,R)$, the proof follows a similar argument as part (a).  Indeed, part (b) in Theorem \ref{thm:difference-w} implies, with probability at least $1-\frac{\gamma}{8}$, that
\begin{equation}\label{eq:stability-wbar-bounded}
    \sup_{S\simeq S'}\delta_{\A}(S,S') \le \tilde{\Delta}_\sgd(\delta/2).
\end{equation}
Note that $\|\bar{\bw}\|_2\le R$ in this case, then combining   \eqref{eq:stability-wbar-bounded} and Lemma \ref{lem:generror-high-probability} with $G=R$ together, with probability at least $1-\frac{\gamma}{4}$, we have
\begin{align*}
    \cR(\Bar{\bw})-\cR_S(\Bar{\bw})
    &\le c \bigg( (M+LR^\alpha) \tilde{\Delta}_\sgd(\delta/2) \log(n)\log({8}/{\gamma}) +  \big( \sup_{z\in \Z} \ell(0,z) + (M+LR^\alpha)R\big) \sqrt{\frac{\log({8}/{\gamma})}{n}}  \bigg) \\
    &=\O\bigg( \tilde{ \Delta}_\sgd(\delta/2) \log(n)\log(1/\gamma)+  \sqrt{\frac{  \log(1/\gamma)}{n}} \bigg),
\end{align*}
where $c>0$ is a constant. This completes the proof of part (b). 
\end{proof}

In the following lemma, we use techniques in optimization theory to control the optimization error $\cR_S(\Bar{\bw})-\cR_S(\bw^{*})$.
\begin{lemma}\label{lem:output-utility-3}
Suppose the loss function $\ell$ is  nonnegative, convex and $\alpha$-H\"older smooth with parameter $L$. Let $\A$ be the SGD 
  with $T$ iterations and  $\eta_t=\eta<\min \{1,1/L\}$ based on the dataset $S=\{z_1,\cdots,z_n\}$, and   $\Bar{\bw}=\frac{1}{T}\sum_{t=1}^T\bw_t$ be the output produced by $\A$.  Then,  for any $\gamma \in (0, 1)$, the following statements hold true. 
\begin{enumerate}[label=({\alph*})]
\item If $\W=\R^d$, then, with probability at least $1-\frac{\gamma}{4}$, there holds
$$\cR_S(\Bar{\bw})-\cR_S(\bw^{*})=\O\bigg(\eta^{\frac{1+\alpha}{2}}T^{\frac{\alpha}{2}} \sqrt{\log(1/\gamma)} + \|\bw^{*}\|_2^{1+\alpha}  \sqrt{\frac{\log(1/\gamma)}{T}}+ \frac{\|\bw^{*}\|_2^2}{\eta T} +  \|\bw^{*}\|_2^{1+\alpha} \eta\bigg).$$
\item If $\W \subseteq \mathcal{B}(0,R)$ with $R> 0$, then, with probability at least $1-\frac{\gamma}{4}$, we have
$$\cR_S(\Bar{\bw})-\cR_S(\bw^{*})=\O\bigg(  \|\bw^{*}\|_2^{1+\alpha} \sqrt{\frac{\log(1/\gamma)}{T}}+ \frac{\|\bw^{*}\|_2^2}{\eta T} +  \|\bw^{*}\|_2^{1+\alpha} \eta\bigg).$$
\end{enumerate}
\end{lemma}
\begin{proof}\noindent (a) We first consider the case $\W=\R^d$. 
From the convexity of $\ell$, we have
\begin{align}\label{eq:opt-decomposition}
    \cR_S(\Bar{\bw})-\cR_S(\bw^{*})
    &\le \frac{1}{T}\sum_{t=1}^T\cR_S(\bw_t)-\cR_S(\bw^{*})\nonumber\\
    &= \frac{1}{T}\sum_{t=1}^T[\cR_S(\bw_t) - \ell(\bw_t,z_{i_t})]+ \frac{1}{T}\sum_{t=1}^T[\ell(\bw^{*},z_{i_t})-\cR_S(\bw^{*})] +\frac{1}{T}\sum_{t=1}^T[ \ell(\bw_t,z_{i_t})-  \ell(\bw^{*},z_{i_t})].
\end{align}
First, we consider the upper bound of $\frac{1}{T}\sum_{t=1}^T[\cR_S(\bw_t) - \ell(\bw_t;z_{i_t})]$.
Since $\{z_{i_t}\} $ is uniformly sampled from the dataset $S$, then for all  $t=1,\ldots,T$ we obtain 
$$\EX_{z_{i_t}}[\ell(\bw_t,z_{i_t})| \bw_1,...,\bw_{t-1}]=\cR_S(\bw_t).$$
By the convexity of $\ell$, the definition of H\"older smoothness and Lemma \ref{lem:boundedness},  for any $z\in \Z$ and all $t\in[T]$, there holds
\begin{align}\label{eq:ell-wt-bound}
    \ell(\bw_t,z) &\le \sup_z \ell(0,z)+\langle \partial \ell(\bw_t,z), \bw_t\rangle \le   \sup_z \ell(0,z)+\| \partial \ell(\bw_t,z)\|_2 \|\bw_t\|_2 \nonumber\\
    &\le \sup_z \ell(0,z)+ (M+L\|\bw_t\|_2^\alpha)\|\bw_t\|_2\le  \sup_z \ell(0,z)+M(C_{\alpha}T\eta)^{\frac{1}{2}}+L(C_{\alpha}T\eta)^{\frac{1+\alpha}{2}}.
\end{align}
Similarly, for any $z\in \Z$, we have
\begin{align}
    \label{eq:ell-w*-bound}
    \ell(\bw^{*},z) \le \sup_z \ell(0;z)+M\|\bw^{*}\|_2+L\|\bw^{*}\|_2^{1+\alpha}.
\end{align}
Now, combining Lemma \ref{lem:azuma} with \eqref{eq:ell-wt-bound} and noting $\eta>1/T$, we get the following inequality with
probability at least $1-\frac{\gamma}{8}$
\begin{align}\label{eq:opt-1}
 \frac{1}{T}\sum_{t=1}^T[\cR_S(\bw_t) - \ell(\bw_t,z_{i_t})]&\le \big(\sup_z \ell(0,z)+M(C_{\alpha}T\eta)^{\frac{1}{2}}+L(C_{\alpha}T\eta)^{\frac{1+\alpha}{2}}\big)\sqrt{\frac{2\log(\frac{8}{\gamma})}{T} } =\O\Big(\eta^{\frac{1+\alpha}{2}}T^{\frac{\alpha}{2}} \sqrt{\log(1/\gamma)} \Big).
\end{align}
According to Lemma \ref{lem:hoeffding}, 
 with probability at least $1-\frac{\gamma}{8}$, there holds
\begin{align}\label{eq:opt-2}
  \frac{1}{T}\sum_{t=1}^T[\ell(\bw^{*};z_{i_t})-\cR_S(\bw^{*})]&\le \big(\sup_z \ell(0,z)+M\|\bw^{*}\|_2+L\|\bw^{*}\|_2^{1+\alpha}\big)\sqrt{\frac{\log({8}/{\gamma})}{2T}} =\O\Big(\|\bw^{*}\|_2^{1+\alpha} \sqrt{\frac{\log(1/\gamma)}{T}} \Big).
\end{align}
Finally, we consider the term $\frac{1}{T}\sum_{t=1}^T [\ell(\bw_t,z_{i_t})-  \ell(\bw^{*},z_t)]$.
The update rule implies $\bw_{t+1}-\bw^{*}=\big(\bw_{t}-\bw^{*}\big)-\eta \partial \ell(\bw_t,z_{i_t})$, from which we know
\begin{align*}
    \|\bw_{t+1}-\bw^{*}\|_2^2&=\|\big(\bw_{t}-\bw^{*}\big)-\eta \partial \ell(\bw_t,z_{i_t})\|_2^2\\
    &=\|\bw_{t}-\bw^{*}\|_2^2 +\eta^2\|\partial \ell(\bw_t,z_{i_t})\|_2^2 -2\eta \langle \partial \ell(\bw_t,z_{i_t}), \bw_t-\bw^{*} \rangle.
\end{align*}
It then follows that
$$\langle \partial \ell(\bw_t,z_{i_t}), \bw_t-\bw^{*}\rangle= \frac{1}{2\eta}\big(\|\bw_t-\bw^{*}\|_2^2- \|\bw_{t+1}-\bw^{*}\|_2^2\big)+\frac{\eta}{2}\| \partial \ell(\bw_t,z_{i_t})\|_2^2.$$
Combining the above inequality and the convexity of $\ell$ together, we derive
\begin{align}\label{eq:opt-3}
     \frac{1}{T}\sum_{t=1}^T [\ell(\bw_t,z_{i_t})-\ell(\bw^{*},z_{i_t})] & \le \frac{1}{T}\sum_{t=1}^T \Bigl[ \frac{1}{2\eta}\big(\|\bw_t-\bw^{*}\|_2^2-\|\bw_{t+1}-\bw^{*}\|_2^2 \big)+\frac{\eta}{2} \| \partial \ell(\bw_t,z_{i_t})\|_2^2\Bigr] \nonumber\\
     &\le \frac{1}{2T\eta}\|\bw_1-\bw^{*}\|_2^2 + \frac{\eta }{2T} \sum_{t=1}^T \| \partial \ell(\bw_t,z_{i_t})\|_2^2.
\end{align}
Since $0\le \frac{2\alpha}{1+\alpha} \le 1$, Lemma~\ref{lem:self-bounding} implies the following inequality for any $t=1,\ldots,T$ 
\[\|\partial \ell(\bw_t;z_{i_t})\|_2^2\le c_{\alpha, 1} \ell^{\frac{2\alpha}{1+\alpha}}(\bw_t;z_{i_t}) \le  c_{\alpha, 1} \max\{\ell(\bw_t;z_{i_t}) , 1 \}\le c_{\alpha, 1}  \ell(\bw_t;z_{i_t})  + c_{\alpha, 1}.\] 
Putting $\|\partial \ell(\bw_t;z_{i_t})\|^2_2\le c_{\alpha, 1}  \ell(\bw_t;z_{i_t})  + c_{\alpha, 1}$ back into \eqref{eq:opt-3} and noting $\|\bw_1\|_2=0$, we have
\begin{align*}
     \frac{1}{T}\sum_{t=1}^T [\ell(\bw_t,z_{i_t})-\ell(\bw^{*},z_{i_t})] & \le \frac{\|\bw^{*}\|_2^2}{2\eta T}+ \frac{c_{\alpha, 1}  \eta }{2T} \sum_{t=1}^T  \ell(\bw_t,z_{i_t})  + \frac{c_{\alpha, 1} \eta}{2}.
\end{align*}
Rearranging the above inequality and using \eqref{eq:ell-w*-bound}, we derive
\begin{align}\label{eq:opt-4}
     \frac{1}{T}\sum_{t=1}^T [\ell(\bw_t,z_{i_t})-\ell(\bw^{*},z_{i_t})] & \le \frac{1}{ 1-\frac{c_{\alpha,1}\eta }{2} } \Bigl(\frac{\|\bw^{*}\|_2^2}{2\eta T}+ \frac{c_{\alpha, 1}  \eta }{2T} \sum_{t=1}^T  \ell(\bw^{*},z_{i_t})  + \frac{c_{\alpha, 1} \eta}{2}\Bigr)\nonumber\\
     &=\O\Big( \frac{\|\bw^{*}\|_2^2}{  \eta T}+ \|\bw^{*}\|_2^{1+\alpha} \eta \Big).
\end{align}

Now, plugging \eqref{eq:opt-1}, \eqref{eq:opt-2} and \eqref{eq:opt-4} back into \eqref{eq:opt-decomposition}, we derive
$$ \cR_S(\Bar{\bw})-\cR_S(\bw^{*})=\O\Big(\eta^{\frac{1+\alpha}{2}}T^{\frac{\alpha}{2}} \sqrt{\log(1/\gamma)} + \|\bw^{*}\|_2^{1+\alpha} \sqrt{\frac{\log(1/\gamma)}{T}}+ \frac{\|\bw^{*}\|_2^2}{\eta T} +  \|\bw^{*}\|_2^{1+\alpha} \eta \Big) $$
with probability at least $1-\frac{\gamma}{4}$, which  completes the proof of part (a).

\ 
 
\noindent (b)  Consider the bounded domain case. Since $\|\bw_t\|_2\le R$ for any $t\in[T]$, then by  the convexity of $\ell$ and the definition of H\"older smoothness,  for any $z\in \Z$, there holds $\ell(\bw_t,z) \le  \sup_z \ell(0,z)+(M+LR^\alpha)R.
$ Combining the above inequality and Lemma \ref{lem:azuma} together, with probability at least $1-\frac{\gamma}{8}$, we obtain
\begin{align}\label{eq:opt-1-bounded}
 \frac{1}{T}\sum_{t=1}^T[\cR_S(\bw_t) - \ell(\bw_t,z_{i_t})]&\le \big(\sup_z \ell(0,z)+(M+LR^\alpha)R\big)\sqrt{\frac{2\log(\frac{8}{\gamma})}{T} } =\O\Big(  \sqrt{ \frac{\log(1/\gamma)}{T}} \Big).  
\end{align}
Since $\|\bw_{t+1}-\bw^{*}\|^2_2=\|\proj_\W\big(\bw_t-\eta \partial \ell(\bw_t,z_{i_t}) \big) -\bw^{*}\|_2^2\le \| (\bw_t -\bw^{*})-\eta \partial \ell(\bw_t,z_{i_t}) \|_2^2$, then \eqref{eq:opt-4} also holds true in this case. 
Putting \eqref{eq:opt-1-bounded}, \eqref{eq:opt-2} and \eqref{eq:opt-4} back into \eqref{eq:opt-decomposition}, with probability at least $1-\frac{\gamma}{4}$, we have  
$$\cR_S(\Bar{\bw})-\cR_S(\bw^{*})=\O\bigg(  \|\bw^{*}\|_2^{1+\alpha} \sqrt{\frac{\log(1/\gamma)}{T}}+ \frac{\|\bw^{*}\|_2^2}{\eta T} +  \|\bw^{*}\|_2^{1+\alpha} \eta\bigg).$$
The proof is completed. 
\end{proof}

Now, we are in a position to prove the utility guarantee for \texttt{DP-SGD-Output} algorithm. First, we give the proof for the unbounded domain case (i.e. Theorem \ref{thm:output-utility}). 

\begin{proof}[Proof of Theorem~\ref{thm:output-utility}]
Note that
$\cR_S(\bw^{*})-\cR(\bw^{*})=\cR_S(\bw^{*})-\EX_S[\cR_S(\bw^{*})]$.
By Hoeffding  inequality and \eqref{eq:ell-w*-bound}, with probability at least $1-\frac{\gamma}{4}$, there holds
\begin{align}\label{eq:decom-difference-5}
    \cR_S(\bw^{*})-\cR(\bw^{*}) \le \Big(\sup_{z\in\Z}\ell(0,z)+M\|\bw^{*}\|_2+ L\|\bw^{*}\|_2^{1+\alpha}\Big)\sqrt{\frac{\log({4}/{\gamma})}{2n}}=\O\Big(\|\bw^{*}\|_2^{1+\alpha}\sqrt{\frac{\log(1/\gamma)}{n}} \Big).
\end{align}

Combining  part (a) in Lemmas~\ref{lem:output-utility-1},  \ref{lem:output-utility-2},  \ref{lem:output-utility-3} and \eqref{eq:decom-difference-5} together, with probability at least $1-\gamma$, the population excess risk can be bounded as follows
\begin{align}\label{eq:output-decom-unbounded}
    &\cR(\bw_{\priv})-\cR(\bw^{*})\nonumber \\
    &=[\cR(\bw_{\priv})-\cR(\Bar{\bw})]+[\cR(\Bar{\bw})-\cR_S(\Bar{\bw})]
    +[\cR_S(\Bar{\bw})-\cR_S(\bw^{*})] +[\cR_S(\bw^{*})-\cR(\bw^{*})]\nonumber\\
    &=\O\bigg( (T\eta)^{\frac{\alpha}{2}} \sigma \sqrt{d}
    \big( \log(1/\gamma) \big)^{\frac{1}{4}} + \sigma^{1+\alpha} d^{\frac{1+\alpha}{2}}    ( \log(1/\gamma) )^{\frac{1+\alpha}{4}}  + (T\eta)^{\frac{\alpha}{2}} \Delta_\sgd(\delta/2) \log(n)\log(1/\gamma)+ \eta^{\frac{1+\alpha}{2}} \Big( T^{\frac{1+\alpha}{2}}\sqrt{\frac{ \log(1/\gamma)}{n}}\nonumber\\
    &\qquad  + T^{\frac{\alpha}{2}}\sqrt{ \log(1/\gamma)} \Big) + \frac{\|\bw^{*}\|_2^2}{\eta T} + \|\bw^{*}\|_2^{1+\alpha}\eta  + \|\bw^{*}\|_2^{1+\alpha}\sqrt{\frac{\log(1/\gamma)}{n}} \bigg).
\end{align}
Plugging $\Delta_\sgd(\delta/2)=\O\Big( \sqrt{T}\eta^{\frac{1}{1-\alpha}}+\frac{(T\eta)^{1+\frac{\alpha}{2}}\log(n/\delta)}{n} \Big)$ and $\sigma=\O(\frac{\sqrt{\log(1/\delta)} \Delta_\sgd(\delta/2) }{\epsilon})$ back into \eqref{eq:output-decom-unbounded}, we have 
 \begin{align}\label{eq:output-decom-unbounded-3}
    &\cR(\bw_{\priv})-\cR(\bw^{*})\nonumber\\
    &= \O\bigg(  T^{\frac{1+\alpha}{2}}\sqrt{\frac{ \log(1/\gamma)}{n}}    \eta^{\frac{1+\alpha}{2}}  + \frac{T^{1+\alpha} \sqrt{d\log(1/\delta)} \big( \log(1/\gamma) \big)^{\frac{1}{4}}  \log(n/\delta) }{ n\epsilon } \eta^{1+\alpha}   +  \frac{ d^{\frac{1+\alpha}{2}}    ( \log(1/\gamma) )^{\frac{1+\alpha}{4}} \big(T \log(\frac{1}{\delta} )\big)^{\frac{1+\alpha}{2}} }{ \epsilon^{1+\alpha} } \eta^{\frac{1+\alpha}{1-\alpha}}   \nonumber\\
    &\qquad +  \frac{\big(d\log(1/\delta)\big)^{\frac{1+\alpha}{2}}    ( \log(1/\gamma) )^{\frac{1+\alpha}{4}} T^{(1+\frac{\alpha}{2})(1+\alpha)} \big( \log(n/\delta )\big)^{ 1+\alpha }}{(n\epsilon)^{1+\alpha}} \eta^{(1+\frac{\alpha}{2})(1+\alpha)} + \frac{T^{\frac{1+\alpha}{2}} \sqrt{d \log(\frac{1}{\delta} )}
    \big( \log(1/\gamma) \big)^{\frac{1}{4}} }{\epsilon} \eta^{\frac{2+\alpha-\alpha^2}{2(1-\alpha)}} \nonumber\\
    &\qquad  +   T^{\frac{1+\alpha}{2}}  \log(n)\log(1/\gamma)  \eta^{\frac{2+\alpha-\alpha^2}{2(1-\alpha)}} + \frac{ T^{1+ \alpha } \log(n/\delta) \log(n)\log(1/\gamma)  }{n}\eta^{1+ \alpha }+  \frac{1}{\eta T} +  \eta +  \sqrt{\frac{\log(1/\gamma)}{n}} \bigg) \cdot \|\bw^{*}\|_2^2.
\end{align}
Taking the derivative of $\frac{1}{T\eta} + T^{\frac{1+\alpha}{2}}\sqrt{\frac{ \log(1/\gamma)}{n}}    \eta^{\frac{1+\alpha}{2}} $ w.r.t $\eta$ and setting it to $0$, then we have  $\eta=  n^{\frac{1}{3+\alpha}}/\big(T(\log(1/\gamma))^{\frac{1}{3+\alpha}}\big) $. Putting this $\eta$ back into \eqref{eq:output-decom-unbounded-3}, we obtain
 \begin{align}\label{eq:output-decom-unbounded-6}
    &\cR(\bw_{\priv})-\cR(\bw^{*})\nonumber
    \\  
    &=  \O\bigg(\frac{ n^{ \frac{(2-\alpha)(1+\alpha)}{2(1-\alpha)(3+\alpha)}} \sqrt{d \log(1/\delta) }}{T^{\frac{1+\alpha}{2(1-\alpha)}} \epsilon \big(
    \log(1/\gamma)\big)^{\frac{ 1+ 4\alpha- \alpha^2 }{4(1-\alpha)(3+\alpha)}} } + \frac{ \sqrt{d\log(1/\delta)}   \log(n/\delta) }{ n^{\frac{2}{3+\alpha}} \epsilon  \big(
    \log(1/\gamma)\big)^{\frac{ 1+ \alpha}{4 (3+\alpha)}}} + \Big( \frac{ \sqrt{d\log(1/\delta)}   \log(n/\delta) }{ n^{\frac{4+\alpha}{2(3+\alpha)}} \epsilon  \big(
    \log(1/\gamma)\big)^{\frac{ 1+\alpha}{4 (3+\alpha)}}}  \Big)^{1+\alpha} \nonumber\\
    &\quad + \Big( \frac{ n^{ \frac{1}{(1-\alpha)(3+\alpha)}} \sqrt{d \log(1/\delta) }}{T^{\frac{1+\alpha}{2(1-\alpha)}} \epsilon \big(
    \log(1/\gamma)\big)^{\frac{ (1+\alpha)^2 }{4(1-\alpha)(3+\alpha)}} }  \Big)^{1+\alpha}    + \log(n) \log(n/\delta)\big(
    \log(1/\gamma)\big)^{\frac{ 2}{ 3+\alpha }} \Big( \frac{n^{\frac{2+\alpha-\alpha^2}{2(3+\alpha)(1-\alpha)}}}{T^{\frac{1+\alpha}{2(1-\alpha)}}} + \frac{1}{n^{  \frac{2}{ 3+\alpha }} } + \frac{ 1 }{ n^{
    \frac{1}{3+\alpha}} } + \frac{ n^{\frac{1}{3+\alpha}} }{T} \Big) \bigg)  \cdot \|\bw^{*}\|_2^2.
\end{align}
To achieve the best rate with a minimal  computational cost, we choose the smallest $T$ such that $  \frac{ n^{ \frac{(2-\alpha)(1+\alpha)}{2(1-\alpha)(3+\alpha)}} }{ T^{\frac{1+\alpha}{2(1-\alpha)}} }  =\O(   \frac{1}{ n^{\frac{2}{3+\alpha}}})$, $  \frac{n^{ \frac{1+\alpha}{(1-\alpha)(3+\alpha)}} }{ T^{\frac{(1+\alpha)^2}{2(1-\alpha)}} }  = \O(   \frac{1}{ n^{\frac{(4+\alpha)(1+\alpha)}{2(3+\alpha)}} }  ) $ and $  \frac{n^{\frac{2+\alpha-\alpha^2}{2(3+\alpha)(1-\alpha)}}}{T^{\frac{1+\alpha}{2(1-\alpha)}}} + \frac{1}{n^{  \frac{2}{ 3+\alpha }} }   + \frac{ n^{\frac{1}{3+\alpha}} }{T}  = \O( \frac{ 1 }{ n^{ \frac{1}{3+\alpha}} } )  $. Hence, we set  $T\asymp  n^{\frac{-\alpha^2-3\alpha+6}{(1+\alpha)(3+\alpha)}} $ if $0\le\alpha\le \frac{\sqrt{73}-7}{4}$, and $T\asymp  n$ else. Now, putting the choice of $T$ back into \eqref{eq:output-decom-unbounded-6}, we derive
\begin{align*}
    \cR(\bw_{\priv})-\cR(\bw^{*})
    =&  \O\bigg( \frac{\sqrt{d\log(1/\delta)} \log(n/\delta)  }{ (\log(1/\gamma))^{\frac{1+ \alpha}{4(3+\alpha)}} n^{\frac{2}{3+\alpha}} \epsilon} +   \Bigl( \frac{\sqrt{d\log(1/\delta)} \log(n/\delta)  }{ (\log(1/\gamma))^{\frac{1+ \alpha}{4(3+\alpha)}} n^{\frac{4+\alpha}{2(3+\alpha)}} \epsilon}  \Bigr)^{1+\alpha}\\&\quad  + \frac{\log(n)\big(\log(1/\gamma)\big)^{\frac{2}{ 3+\alpha }}\log(n/\delta)}{n^{\frac{1}{3+\alpha}}}\bigg) \cdot  \|\bw^{*}\|_2^2.
\end{align*}
Without loss of generality, we assume the first term   of the above utility bound is less than 1. Therefore, with probability at least $1-\gamma$, there holds
\begin{align*}
    \cR(\bw_{\priv})-\cR(\bw^{*})
    =  \|\bw^{*}\|_2^2\cdot  \O\bigg( \frac{\sqrt{d\log(1/\delta)} \log(n/\delta)  }{ (\log(1/\gamma))^{\frac{1+ \alpha}{4(3+\alpha)}} n^{\frac{2}{3+\alpha}} \epsilon} + \frac{\log(n)\big(\log(1/\gamma)\big)^{\frac{2}{ 3+\alpha }}\log(n/\delta)}{n^{\frac{1}{3+\alpha}}}\bigg).
\end{align*}
The proof is completed. 
\end{proof}

\bigskip

Finally, we provide the proof of utility guarantee  for the  \texttt{DP-SGD-Output} algorithm  when $\W\subseteq \mathcal{B}(0, R)$ (i.e. Theorem \ref{thm:output-utility-proj}).

\begin{proof}[Proof of Theorem~\ref{thm:output-utility-proj}] 
 The proof is similar to that of Theorem~\ref{thm:output-utility}.  
Indeed, plugging  part (b) in Lemmas~\ref{lem:output-utility-1},  \ref{lem:output-utility-2},  \ref{lem:output-utility-3} and  \eqref{eq:decom-difference-5} back into  \eqref{eq:error-decom1}, with probability at least  $1-\gamma$, the population excess risk can be bounded as follows 
\begin{align*}
    \cR(\bw_{\priv})-\cR(\bw^{*})
     &= \O\Big(   \sigma\sqrt{d} (\log(1/\gamma))^{\frac{1}{4}} + \sigma^{1+\alpha}d^{\frac{1+\alpha}{2}}(\log(1/\gamma))^{\frac{1+\alpha}{4}} + \tilde{\Delta}_{\sgd}(\delta/2) \log(n)\log(1/\gamma) + \sqrt{ \frac{\log(1/\gamma)}{n} } \nonumber\\  
     &\qquad  +  \frac{\|\bw^{*}\|_2^2}{T\eta}   + \|\bw^{*}\|_2^{1+\alpha}  \eta + \|\bw^{*}\|_2^{1+\alpha} \sqrt{\frac{\log(1/\gamma)}{n}} \Big).
\end{align*}
Note that $\tilde{\Delta}_{\sgd}(\delta/2)=\O( \sqrt{T}\eta^{\frac{1}{1-\alpha}} + \frac{ T \eta  \log(n/\delta)}{ n } )$ and $\sigma=\frac{\sqrt{2\log(2.5/\delta)} \tilde{\Delta}_{\sgd}(\delta/2)}{\epsilon}$. Then we have
\begin{align}\label{eq:decom-psgd-output}
  & \cR(\bw_{\priv})-\cR(\bw^{*}) \nonumber \\   &
    =   \O\biggl(  \big( \frac{ T \log(n/\delta) ) \log(n) \log(1/\gamma)}{ n} +  \frac{  T\sqrt{d\log(1/\delta)}  \log(n/\delta)  (\log(1/\gamma))^{\frac{1}{4}}}{ n\epsilon}  \big) \eta   \nonumber \\
    & \qquad   +  \big( \frac{ \sqrt{ \log(1/\delta) T d } (\log(1/\gamma))^{\frac{1}{4}} }{ \epsilon } + \sqrt{T} \log(n)\log(1/\gamma)  \big) \eta^{\frac{1}{1-\alpha}} + \frac{(Td \log(1/\delta)  )^{\frac{1+\alpha}{2}} (\log(1/\gamma))^{\frac{1+\alpha}{4}} }{\epsilon^{1+\alpha}} \eta^{\frac{1+\alpha}{1-\alpha}}  \nonumber \\
    &\qquad  + \big( \frac{  T \sqrt{d\log(1/\delta)}  \log(n/\delta) (\log(1/\gamma))^{\frac{1}{4}} }{n\epsilon} \big)^{1+\alpha} \eta^{1+\alpha}     + \frac{1}{T\eta} + \sqrt{\frac{\log(1/\gamma)}{n}} \biggr)  \cdot \|\bw^{*}\|_2^2.
\end{align}
Consider the tradeoff between  $1/\eta$ and $\eta $. Taking the derivative of $\big( \frac{  T \log(n/\delta) \log(n) \log(1/\gamma)}{ n}+ \frac{ T\sqrt{d\log(1/\delta)}  \log(n/\delta)  (\log(1/\gamma))^{1/4}}{ n\epsilon} \big) \eta  $ $+\frac{1}{T\eta}$ w.r.t $\eta$ and setting it to $0$, we have 
$\eta=1 / \Big(T \max\Big\{ \frac{  \sqrt{\log(n/\delta) \log(n) \log(1/\gamma)} }{\sqrt{n}}, \frac{   \big(d\log(1/\delta)\big)^{1/4} \sqrt{\log(n/\delta)} (\log(1/\gamma))^{1/8}}{\sqrt{n\epsilon}}  \Big\} \Big)$.  
Then putting the value of $\eta$ back into \eqref{eq:decom-psgd-output}, we obtain
\begin{align*}
     &\cR(\bw_{\priv})-\cR(\bw^{*}) \\
      & =    \O\bigg(  \frac{\big(d\log(1/\delta)\big)^{\frac{1}{4}} (\log(1/\gamma))^{\frac{1 }{8}} \sqrt{\log(n/\delta)} }{ \sqrt{n\epsilon} }  +  \Bigl(\frac{ \big(d\log(1/\delta)\big)^{\frac{1}{4}} (\log(1/\gamma))^{\frac{1 }{8}} \sqrt{\log(n/\delta)}}{ \sqrt{n\epsilon} }\Bigr)^{1+\alpha}\\ 
     & \quad  + \frac{\big(d\log(1/\delta) \big)^{\frac{1-2\alpha}{4(1-\alpha)}} (\log(1/\gamma))^{\frac{1-2\alpha}{8(1-\alpha)}} n^{\frac{1}{2(1-\alpha)}} \epsilon^{\frac{2\alpha-1}{2(1-\alpha)}}}{ T^{\frac{1+\alpha}{2(1-\alpha)}} (\log(n/\delta))^{\frac{ 1 }{2(1-\alpha)}}  }  + \Big(\frac{\big(d\log(1/\delta) \big)^{\frac{1-2\alpha}{4(1-\alpha)}} (\log(1/\gamma))^{\frac{1-2\alpha}{8(1-\alpha)}} n^{\frac{1}{2(1-\alpha)}} \epsilon^{\frac{2\alpha-1}{2(1-\alpha)}}}{ T^{\frac{1+\alpha}{2(1-\alpha)}} (\log(n/\delta))^{\frac{ 1 }{2(1-\alpha)}}  }\Big)^{1+\alpha} \\ 
     &  \quad +  \sqrt{\log(n)\log(1/\gamma)\log(1/\delta)} \big( \frac{1}{\sqrt{n}} + \frac{n^{\frac{1}{2(1-\alpha)}}}{ T^{\frac{1+\alpha}{2(1-\alpha)}}} \big)  \bigg)  \cdot  \|\bw^{*}\|_2^2. 
\end{align*} 
Similarly, we choose the smallest $T$ such that $ \frac{   n^{\frac{1}{2(1-\alpha)}} }{ T^{\frac{1+\alpha}{ 2(1-\alpha)}} } = \O(\frac{1}{\sqrt{n}} )$. Hence, we set  $T\asymp  n^{\frac{2-\alpha}{1+\alpha}} $ if $\alpha< \frac{1}{2}$, and $T\asymp n$ else.  Since $\frac{1}{4} \ge \frac{1-2\alpha}{2(1-\alpha)}$, we have
\begin{align*}
    \cR(\bw_{\priv})-\cR(\bw^{*})=     \O\bigg(&  \frac{\big(d\log(1/\delta)\big)^{\frac{1}{4}} (\log(1/\gamma))^{\frac{1 }{8}} \sqrt{\log(n/\delta) } }{ \sqrt{n\epsilon} }  +  \bigl( \frac{\big(d\log(1/\delta)\big)^{\frac{1}{4}} (\log(1/\gamma))^{\frac{1 }{8}} \sqrt{\log(n/\delta) } }{ \sqrt{n\epsilon} }\bigr)^{1+\alpha} \\ 
     &   + \frac{\sqrt{\log(n)\log(1/\gamma)\log(n/\delta)} }{\sqrt{n}}   \bigg)  \cdot \|\bw^{*}\|_2^2. 
\end{align*}
It is reasonable to assume the first term is less than $1$ here. Therefore, with probability at least $1-\gamma$, there holds 
\begin{align*}
    \cR(\bw_{\priv})-\cR(\bw^{*})=    \|\bw^{*}\|_2^2\cdot    \O\Big(  \frac{\big(d\log(1/\delta)\big)^{\frac{1}{4}} (\log(1/\gamma))^{\frac{1 }{8}} \sqrt{\log(n/\delta) } }{ \sqrt{n\epsilon}  }   + \frac{\sqrt{\log(n)\log(1/\gamma)\log(n/\delta)} }{\sqrt{n}}   \Big).
\end{align*}
The proof is completed.  
\end{proof}

\subsection{Proofs on Differential Privacy of SGD with Gradient Perturbation}\label{subsec:gradient}

We now turn to the analysis for \texttt{DP-SGD-Gradient} algorithm (i.e. Algorithm \ref{alg2}) and provide the proofs for Theorems \ref{thm:dpSGD-gradient-Holder} and \ref{thm:gradient-utility-proj}. We start with the proof of  Theorem \ref{thm:dpSGD-gradient-Holder} on the privacy guarantee for Algorithm \ref{alg2}.

\begin{proof}[Proof of Theorem~\ref{thm:dpSGD-gradient-Holder}]

Consider the mechanism $\mathcal{G}_t=\mathcal{M}_t+\bb_t$, where $\mathcal{M}_t=\partial \ell(\bw_t,z_{i_t})$. For any $\bw_{t } \in \W$ and any $z_{i_{t }}, z'_{i_{t }} \in \Z$, the definition of $\alpha$-H\"older smoothness implies that
$$ \|\partial\ell(\bw_{t },z_{i_{t }})-\partial\ell(\bw_{t },z'_{i_{t }})\|_2\le 2\big(M+L\|\bw_t\|^\alpha_2\big)\le 2(M+L R^\alpha).$$
Therefore, the $\ell_2$-sensitivity of $\mathcal{M}_t$ is $2(M+L R^\alpha)$. 
Let $$\sigma^2=\frac{14 (M+LR^\alpha)^2 T }{\beta n^2 \epsilon} \Big( \frac{\log(1/\delta)}{(1-\beta)\epsilon} +1\Big).$$
Lemma~\ref{lem:uniform} with $p=\frac{1}{n}$ implies that $\mathcal{G}_t$ satisfies $\Big(\lambda, \frac{\lambda \beta \epsilon}{T \big(\frac{\log(1/\delta)}{(1-\beta)\epsilon}+1\big)}\Big)$-RDP if the following conditions hold
\begin{align}\label{RDPcondition1}
    \frac{\sigma^2}{4(M+LR^\alpha)^2}\ge 0.67
\end{align}
and 
\begin{align}\label{RDPcondition2}
\lambda -1 \le \frac{\sigma^2}{6(M+LR^\alpha)^2} \log\Big(\frac{n}{\lambda(1+ \frac{\sigma^2}{4(M+LR^\alpha)^2} )}\Big).
\end{align}
Let  $\lambda=\frac{\log(1/\delta)}{(1-\beta)\epsilon}+1$. We obtain that $\mathcal{G}_t$ satisfies $(\frac{\log(1/\delta)}{(1-\beta)\epsilon}+1, \frac{\beta \epsilon}{T})$-RDP. Then by the post-processing property of DP (see Lemma \ref{lemma:post-processing}), we know $\bw_{t+1}$ also satisfies $(\frac{\log(1/\delta)}{(1-\beta)\epsilon}+1, \frac{\beta \epsilon}{T})$-RDP for any $t=0,...,T-1$. Furthermore, according to the adaptive composition theorem of RDP (see Lemma~\ref{lem:composition_RDP}), Algorithm~\ref{alg2} satisfies $(\frac{\log(1/\delta)}{(1-\beta)\epsilon}+1, \beta \epsilon)$-RDP. Finally, by Lemma~\ref{lemma:RDP_to_DP}, the output of Algorithm~~\ref{alg2} satisfies $(\epsilon,\delta)$-DP as long as \eqref{RDPcondition1} and \eqref{RDPcondition2} hold.  
\end{proof}

Now, we turn to the generalization analysis of Algorithm~\ref{alg2}. 
First, we estimate the generalization error $\cR(\bw_{\priv})-\cR_S(\bw_{\priv})$ in \eqref{eq:Gradient-error-decom2}. 

\begin{lemma}\label{lem:gradient-utility-1}
 Suppose the loss function $\ell$ is nonnegative, convex  and $\alpha$-H\"older smooth with parameter $L$. Let $\bw_\priv$ be the output produced by Algorithm~\ref{alg2} based on  $S=\{z_1,\cdots,z_n \}$ with $\eta_t=\eta< \min \{1,1/L\}$. Then for any $\gamma \in (0, 1)$, with probability at least $1-\frac{\gamma}{3}$, there holds
 \begin{align*}
    \cR(\bw_{\priv})-\cR_S(\bw_{\priv}) 
    &=\O\Big(  \tilde{ \Delta}_\sgd(\gamma/6) \log(n)\log(1/\gamma) +   \sqrt{\frac{\log(1/\gamma)}{n}} \Big).
\end{align*}
\end{lemma}
\begin{proof}
Part (b) in Theorem~\ref{thm:difference-w} implies that $\tilde{ \Delta}_\sgd(\gamma/6)=\O\Big( \sqrt{T}\eta^{\frac{1}{1-\alpha}}+\frac{  T \eta  \log( n /\gamma)}{n} \Big)$ with probability at least $1-\frac{\gamma}{6}$.
Since the noise  added to the gradient in each iteration is the same for the neighboring datasets $S$ and $S'$, the noise addition does not impact the stability analysis. Therefore, the UAS bound of the noisy SGD is equivalent to the SGD.  
According to Lemma~\ref{lem:generror-high-probability} and $\|\bw_\priv\|_2\le R$, we derive the following inequality with probability at least $1-(\frac{\gamma}{6}+\frac{\gamma}{6})$
\begin{align*}
    \cR(\bw_{\priv})-\cR_S(\bw_{\priv}) 
    &\le c\Big((M+LR^\alpha) \tilde{ \Delta}_\sgd(\gamma/6) \log(n)\log(6/{\gamma}) + \big(M_0 + (M+ LR^\alpha\big)R \sqrt{\frac{\log(6/{\gamma})}{n}} \Big)\nonumber\\
    &=\O\Big(  \tilde{ \Delta}_\sgd(\gamma/6) \log(n)\log(1/\gamma) +  \sqrt{\frac{\log(1/\gamma)}{n}} \Big),
\end{align*}
where $c>0$ is a constant. The proof is completed.
\end{proof}

The following lemma gives an upper bound for the second term $\cR_S(\bw_{\priv})-\cR_S(\bw^{*})$ in \eqref{eq:Gradient-error-decom2}. 
\begin{lemma}\label{lem:gradient-utility-2}
 Suppose the loss function $\ell$ is nonnegative, convex  and $\alpha$-H\"older smooth with parameter $L$. Let $\bw_\priv$ be the output produced by Algorithm~\ref{alg2} based on  $S=\{z_1,\cdots,z_n \}$ with $\eta_t=\eta< \min \{1,1/L\}$. Then, for any $\gamma \in (18\exp(-dT/8), 1)$, with probability at least $1-\frac{\gamma}{3}$, there holds
 \begin{align*}
    \cR_S(\bw_\priv)-\cR_S(\bw^{*})=\O\Big( &  \|\bw^{*}\|_2^{1+\alpha} \sqrt{\frac{\log(1/\gamma)}{T}}  +  \frac{\|\bw^{*}\|_2^2}{T\eta} +  \eta  +  \frac{ \sqrt{\log(1/\delta)\log(1/\gamma)}( \|\bw^{*}\|_2+ \eta)}{n \epsilon}\\&+ \frac{ \eta T d \log(\frac{1}{
\delta})\sqrt{\log(\frac{1}{
\gamma})} }{n^2\epsilon^2 }  \Big).
\end{align*}
\end{lemma}
\begin{proof}
To estimate the term $\cR_S(\bw_{\priv})-\cR_S(\bw^{*})$,  we decompose it as
\begin{align}\label{eq:utility-gradient-3}
    \cR_S(\bw_\priv)-\cR_S(\bw^{*})&\le  \frac{1}{T}\sum_{t=1}^T[\cR_S(\bw_t) - \ell(\bw_t,z_{i_t})]+ \frac{1}{T}\sum_{t=1}^T[\ell(\bw^{*},z_{i_t})-\cR_S(\bw^{*})]+\frac{1}{T}\sum_{t=1}^T[ \ell(\bw_t,z_{i_t})-  \ell(\bw^{*},z_{i_t})].   
\end{align}
Similar to the analysis in \eqref{eq:ell-wt-bound} and \eqref{eq:ell-w*-bound},  we have $\ell(\bw^{*},z)=\O(\|\bw^{*}\|_2^{1+\alpha})$ for all $z\in \Z$ and  $\ell(\bw_t,z)=\O(R + R^{1+\alpha})$ for all $t=1,\ldots,T$ and $z\in \Z$.
Therefore, Azuma-Hoeffding inequality (see Lemma \ref{lem:azuma}) yields, with probability at least $1-\frac{\gamma}{9}$, that
\begin{align}\label{eq:utility-gradient-3-1}
    &\frac{1}{T}\sum_{t=1}^T[\cR_S(\bw_t) - \ell(\bw_t,z_t)] \le \big( \sup_{z\in{ \Z} }\ell(0,z) + \sup_{t=1,\ldots,T;  z\in \Z} \ell(\bw_{t},z) \big) \sqrt{\frac{\log({9}/{\gamma})}{2T}} = \O\Big((R+ R^{ 1+ \alpha}) \sqrt{\frac{\log(1/\gamma)}{T}}
    \Big).
\end{align}
In addition, Hoeffding  inequality (see Lemma \ref{lem:hoeffding}) implies, with probability at least $1-\frac{\gamma}{9}$, that 
\begin{align}\label{eq:utility-gradient-3-2}
    \frac{1}{T}\sum_{t=1}^T[\ell(\bw^{*},z_{i_t})-\cR_S(\bw^{*})]\le (\sup_{z\in{ \Z}} \ell(0,z)+\sup_{z\in{ \Z}} \ell(\bw^{*},z))\sqrt{\frac{\log({9}/{\gamma})}{2T}}= \O\Big(  \|\bw^{*}\|_2^{1+\alpha}  \sqrt{\frac{\log(1/\gamma)}{T}} \Big). 
\end{align} 
Finally, we try to bound $\frac{1}{T}\sum_{t=1}^T [\ell(\bw_t,z_{i_t})-  \ell(\bw^{*},z_{i_t})]$. The SGD update rule implies that $\|\bw_{t+1}-\bw^{*}\|_2^2=\|\proj_\W \big(\bw_{t }-\eta(\partial \ell(\bw_t,z_{i_t}) +\bb_t )\big) -\bw^{*}\|_2^2 \le \|(\bw_{t }-\bw^{*})-\eta(\partial \ell(\bw_t,z_{i_t}) +\bb_t )\|_2^2$, then we have $\langle \bw_t -\bw^{*}, \partial\ell(\bw_t,z_{i_t}) \rangle\le  \frac{1}{2\eta} \big( \|\bw_{t }-\bw^{*}\|_2^2-\|\bw_{t+1}-\bw^{*}\|_2^2  \big)+\frac{\eta}{2}\big(\|\partial\ell(\bw_t,z_{i_t})\|_2^2 +\|\bb_t\|_2^2   \big) - \langle \bb_t, \bw_t-\bw^{*}-\eta \partial \ell(\bw_t,z_{i_t}) \rangle$. Further,  noting $\|\bw_1\|_2=0$, then by the convexity of $\ell$ we have
\begin{align*}
    &\frac{1}{T}\sum_{t=1}^T [\ell(\bw_t,z_{i_t})-  \ell(\bw^{*},z_{i_t})]\le\frac{\|\bw^{*}\|_2^2}{2T\eta}+\frac{\eta}{2T}\sum_{t=1}^T\|\partial \ell(\bw_t,z_{i_t})  \|_2^2- \frac{1}{T}\sum_{t=1}^T \langle \bb_t, \bw_t-\bw^{*}-\eta \partial \ell(\bw_t,z_{i_t})  \rangle +\frac{\eta}{2T}\sum_{t=1}^T\|\bb_t\|_2^2.
\end{align*}
The definition of $\alpha$-H{\"o}lder smoothness implies that $\|\partial \ell(\bw_t,z_{i_t}) \|_2\le M+ L\|\bw_t\|^{\alpha}_2\le M+LR^\alpha$ for any $t$. Then, there hold
\begin{align*}
    \frac{\eta}{2T}\sum_{t=1}^T\|\partial \ell(\bw_t,z_{i_t})  \|_2^2 &\le \frac{\eta}{2T}\sum_{t=1}^T(M+L\|\bw_t \|^{\alpha}_2)^2 =\O(\eta ),
\end{align*}
and
$$\|\bw_t-\bw^{*}-\eta \partial \ell(\bw_t,z_{i_t})\|_2\le \|\bw^{*}\|_2+R+ \eta(M+L R^\alpha).$$
Since $\bb_t$ is an $\sigma^2$-sub-Gaussian random vector, $\frac{1}{T}\langle \bb_t, \bw_t-\bw^{*}-\eta \partial \ell(\bw_t,z_{i_t}) \rangle$ is an $\frac{\sigma^2}{T^2} \big(\|\bw^{*}\|_2+R+ \eta(M+L R^\alpha)\big)^2$-sub-Gaussian random vector. Note that the sub-Gaussian parameter $\frac{\sigma^2}{T^2} \big(\|\bw^{*}\|_2+R+ \eta(M+L R^\alpha)\big)^2$ is   independent of $\bw_{t-1}$ and $\bb_{t-1}$.   Hence, $\frac{1}{T}\sum_{t=1}^T \langle \bb_t, \bw_t-\bw^{*}-\eta \partial \ell(\bw_t,z_{i_t})  \rangle$ is an $\frac{  \sigma^2\sum_{t=1}^T (\|\bw^{*}\|_2+R+ \eta(M+L R^\alpha) )^2 }{T^2}$-sub-Gaussian random vector. Since $\sigma^2=\O(\frac{ T \log(1/\delta)}{n^2 \epsilon^2})$, the tail bound of Sub-Gaussian variables (see Lemma~\ref{lem:tailbound-subG}) implies, with probability at least $1-\frac{\gamma}{18}$, that
\begin{align*}
     \frac{1}{T}\sum_{t=1}^T \langle \bb_t, \bw_t-\bw^{*}-\eta \partial \ell(\bw_t,z_{i_t})  \rangle  
     &\le \frac{\Big( \sigma^2 \big( \|\bw^{*}\|_2 + R+ \eta(M+L R^\alpha) \big)^2\Big)^{\frac{1}{2}} }{\sqrt{T}} \sqrt{2 \log({18}/{\gamma})}\nonumber\\
     &=\O\Big( \sigma (\|\bw^{*}\|_2 + \eta  )\sqrt{\frac{\log(1/\gamma)}{T}} \Big)=\O\Big( \frac{ \sqrt{\log(1/\delta)\log(1/\gamma)}( \|\bw^{*}\|_2+ \eta)}{n \epsilon} \Big).
\end{align*}
According to the Chernoff bound for the $\ell_2$-norm of Gaussian vector with $\mathbf{X}=[\bb_{11},...,\bb_{1d},\bb_{21}...,\bb_{Td}]\in \R^{Td}$(see Lemma~\ref{lem:chernoff-gaussian}), for any $\gamma \in (18\exp(-dT/8), 1)$, with probability at least $1-\frac{\gamma}{18}$, there holds
$$ \frac{\eta}{2T}\sum_{t=1}^T\|\bb_t\|_2^2\le \frac{\eta d}{2 T} \Big(1+(\frac{1}{d}\log({18}/{\gamma}))^{\frac{1}{2}}\Big) T \sigma^2=\O\Big( \frac{  \eta T d \log(\frac{1}{
\delta})\sqrt{\log(\frac{1}{
\gamma})} }{n^2\epsilon^2 } \Big).$$
Therefore, with probability at least $1-\frac{\gamma}{9}$, there holds
\begin{align}\label{eq:utility-gradient-3-3}
    \frac{1}{T}\sum_{t=1}^T[ \ell(\bw_t,z_{i_t})-  \ell(\bw^{*},z_{i_t})] \le  \O\Big(  \frac{\|\bw^{*}\|_2^2}{T\eta} +  \eta  +  \frac{ \sqrt{\log(1/\delta)\log(1/\gamma)}( \|\bw^{*}\|_2+ \eta)}{n \epsilon}  + \frac{ \eta T d \log(1/\delta)\sqrt{\log(1/\gamma)} }{n^2\epsilon^2 }  \Big).
\end{align}
Putting \eqref{eq:utility-gradient-3-1},  \eqref{eq:utility-gradient-3-2} and \eqref{eq:utility-gradient-3-3} back into \eqref{eq:utility-gradient-3}, we obtain, with probability at least $1-\frac{\gamma}{3}$, that
\begin{align*}
   & \cR_S(\bw_\priv)-\cR_S(\bw^{*})\\&
    =\O\Big(     \|\bw^{*}\|_2^{1+\alpha}  \sqrt{\frac{\log(1/\gamma)}{T}}  +  \frac{\|\bw^{*}\|_2^2}{T\eta} +  \eta  +  \frac{ \sqrt{\log({1}/{\delta})\log(1/\gamma)}( \|\bw^{*}\|_2+ \eta)}{n \epsilon} + \frac{ \eta T d \log({1}/{
\delta})\sqrt{\log({1}/{
\gamma})} }{n^2\epsilon^2 }  \Big).
\end{align*}
The proof is completed.
\end{proof}

Now, we are ready to prove the utility theorem for \textit{DP-SGD-Gradient} algorithm.

\begin{proof}[Proof of Theorem~\ref{thm:gradient-utility-proj}]

The Hoeffding inequality implies, with probability at least $1-\frac{\gamma}{3}$, that
\begin{align*}
    \cR_S(\bw^{*})-\cR(\bw^{*}) \le \big(\sup_{z\in \Z} \ell(0,z)+\sup_{z\in {\Z}} \ell(\bw^{*},z) \big) \sqrt{\frac{\log({3}/{\gamma})}{2n}}= \O\Big( \|\bw^{*}\|_2^{1+\alpha}\sqrt{\frac{\log(1/\gamma)}{n}} \Big).
\end{align*}
Combining Lemma~\ref{lem:gradient-utility-1}, Lemma~\ref{lem:gradient-utility-2} and the above inequality together, with probability at least $1-\gamma$, we obtain
\begin{align*}
     \cR(\bw_{\priv})-\cR(\bw^{*})= \O\Big( & \tilde{ \Delta}_\sgd(\gamma/6) \log(n)\log(1/\gamma)  +  \frac{\|\bw^{*}\|_2^2}{T\eta} +  \eta  +  \frac{ \sqrt{\log(1/\delta)\log(1/\gamma)}( \|\bw^{*}\|_2+ \eta  )}{n \epsilon} \nonumber\\
& + \frac{ \eta T d \log(\frac{1}{
\delta})\sqrt{\log(\frac{1}{
\gamma})} }{n^2\epsilon^2 }  + \|\bw^{*}\|_2^{1+\alpha} \sqrt{\frac{\log(1/\gamma)}{n}} \Big). 
\end{align*}
Now, putting $\tilde{ \Delta}_\sgd(\gamma/6) =\O( \sqrt{T}\eta^{\frac{1}{1-\alpha}} + \frac{ T \eta  \log(n/\gamma) }{ n } )$ back into the above estimate, we have 
\begin{align}\label{eq:utility-gradient-5}
    &\cR(\bw_{\priv})-\cR(\bw^{*})\nonumber\\
   & =  \O\Big(   \sqrt{T} \log(n)\log(1/\gamma)\eta^{\frac{1}{1-\alpha}}  +  \frac{\|\bw^{*}\|_2^2}{T\eta}  +  \eta \Big( \frac{ T d \log(1/\delta) \sqrt{\log(1/\gamma)} }{ n^2\epsilon^2 } + \frac{   T \log(n)  \log(n/\gamma)\log(1/\gamma) }{ n }  \Big)\nonumber \\
    &\qquad  +  \|\bw^{*}\|_2^{1+\alpha} \sqrt{\frac{\log(1/\gamma)}{n}}  + \frac{\|\bw^{*}\|_2\sqrt{\log(1/\delta)\log(1/\gamma)}}{n\epsilon}\Big).
\end{align}
To choose a suitable $\eta$ and $T$ such that the algorithm achieves the optimal rate, we consider  the trade-off between $1/\eta$ and $\eta$. We take the derivative of $\frac{1}{T\eta}+\eta \big( \frac{ T d \log(1/\delta) \sqrt{\log( 1/\gamma )} }{ n^2\epsilon^2 } + \frac{ T \log(n) \log(n/\gamma)\log(1/\gamma )  }{ n }   \big)$ w.r.t $\eta$ and set it to $0$, then we have 
$\eta=1/ T \cdot \max\big\{ \frac{\sqrt{ \log(n)\log(n/\gamma) \log(1/\gamma) } }{\sqrt{n}}, \frac{  \sqrt{d \log(1/\delta)} (\log(1/\gamma))^{\frac{1}{4}}}{n\epsilon} \big\} $. 
Putting the value of $\eta$ back into \eqref{eq:utility-gradient-5}, we obtain
\begin{align*}
    &\cR(\bw_{\priv})-\cR(\bw^{*})\\
    &=  \O\bigg(  \frac{ \big(\log(n)\log(1/\gamma)\big)^{\frac{1-2\alpha}{2(1-\alpha)}}n^{\frac{1}{2(1-\alpha)}}}{ T^{\frac{1+\alpha}{2(1-\alpha)}} (\log(n/\gamma))^{\frac{1}{ 2(1-\alpha) }} } +  \frac{  \sqrt{d\log(1/\delta)\log(1/\gamma)}  }{n\epsilon} + \frac{ \sqrt{\log(n)\log(n/\gamma) \log(1/\gamma) }    }{ \sqrt{n} }  \bigg)   \cdot   \|\bw^{*}\|_2^{2 }. 
\end{align*}
In addition, if   $n = \O ( T^{\frac{1+\alpha}{2-\alpha}} )$, then there holds
\begin{align*}
    \cR(\bw_{\priv})-\cR(\bw^{*})= & \|\bw^{*}\|_2^{2 } \cdot\O\Big(   \frac{ \sqrt{d\log(1/\delta) \log(1/\gamma) }}{n\epsilon} + \frac{\sqrt{\log(n)\log(n/\gamma) \log(1/\gamma) }    }{ \sqrt{n} } \Big).
\end{align*}
The above bound matches the optimal rate $\O\big( \frac{\sqrt{d \log(1/\delta)}}{n\epsilon} + \frac{1}{\sqrt{n}} \big)$. 
Furthermore, we want the algorithm to achieve the optimal rate with a low computational cost. Therefore, we set $T \asymp  n^{
\frac{2-\alpha}{1+\alpha}} $ if $\alpha < \frac{1}{2}$, and $T \asymp n $ else.
The proof is completed.

\end{proof}

Finally, we give the proof of Lemma~\ref{lem:condition-beta} on the existence of $\beta$ for Algorithm \ref{alg2} to be $(\epsilon,\delta)$-DP. 

\begin{proof}[Proof of Lemma~\ref{lem:condition-beta}]
We give sufficient conditions for the existence of $\beta\in(0,1)$ such that RDP conditions  \eqref{RDPcondition1} and \eqref{RDPcondition2} hold  with $\sigma^2=\frac{14 ( M+LR^\alpha )^2   \lambda}{\beta n  \epsilon}$ and $\lambda=\frac{2\log(n)}{(1-\beta)\epsilon}+1$ in Theorem~\ref{thm:dpSGD-gradient-Holder}. 
Condition \eqref{RDPcondition1} with $T=n$ and $\delta=\frac{1}{n^2}$ is equivalent to
\begin{align}\label{eq:condition-rdp-1}
     f(\beta):=\beta^2 -\Big(1+\frac{7}{1.34 n \epsilon}\Big)\beta + \frac{7 (2\log(n)+\epsilon)}{1.34n \epsilon^2}\ge 0.
\end{align}
If  $\big( 1+\frac{7 }{1.34 n  \epsilon} \big)^2 < \frac{28 (2\log(n)+\epsilon)}{1.34n \epsilon^2}$, then $f(\beta)\ge 0$ for all $\beta$. Then  \eqref{RDPcondition1} holds for any $\beta \in (0,1)$. 
If   $\big( 1+\frac{7 }{1.34 n  \epsilon} \big)^2 \ge \frac{28 (2\log(n)+\epsilon)}{1.34n \epsilon^2}$,  then $\beta \in (0,\beta_1] \cup [\beta_2, + \infty)$ such that the above condition holds, where  $\beta_{1,2}=\frac{1}{2}\Big(\big( 1+\frac{7 }{1.34n \epsilon} \big)\mp\sqrt{\big( 1+\frac{7 }{1.34n  \epsilon} \big)^2-\frac{28 (2\log(n)+\epsilon)}{1.34n  \epsilon^2}}  \Big)$ are two roots of $f(\beta)=0$.

Now, we consider the second RDP condition. Plugging $\sigma^2=\frac{14 ( M+LR^\alpha )^2   \lambda}{\beta n  \epsilon}$ back into  \eqref{RDPcondition2}, we derive 
\begin{align}\label{eq:condition-lambd-1}
    \frac{3\beta n  \epsilon (\lambda-1)}{7 \lambda }+\log(\lambda)+\log(1+\frac{7 \lambda}{2\beta n  \epsilon})\le \log(n).
\end{align}
To guarantee \eqref{eq:condition-lambd-1}, it suffices that the following three inequalities hold
\begin{equation}
    \label{eq:condition-lambda-2}
    \frac{3\beta n  \epsilon (\lambda-1)}{7 \lambda }\le \frac{ \log(n)}{3},
\end{equation}
\begin{equation}\label{eq:condition-lambda-3}
   \log(\lambda)\le \frac{ \log(n)}{3} ,
\end{equation}
\begin{equation}\label{eq:condition-lambda-4}
     \log \big(1+\frac{7 \lambda}{2\beta n \epsilon} \big)\le  \frac{ \log(n)}{3}.
\end{equation}
We set $\lambda=\frac{2\log(n)}{(1-\beta)\epsilon}+1$ in the above three inequalities. Since $\lambda> 1$, then \eqref{eq:condition-lambda-2} holds if $ \beta \le 7  \log(n)/9n  \epsilon$. 
Eq. \eqref{eq:condition-lambda-3}  reduces to $\beta \le 1- \frac{2\log(n)}{(n^{1/3}-1)\epsilon} $. Moreover,  \eqref{eq:condition-lambda-4}  is equivalent to the following inequality
\begin{align}\label{eq:condition-rdp-2}
    g(\beta):=\beta^2-(1+\frac{7 }{2n (n^{\frac{1}{3}}-1)\epsilon})\beta +\frac{7 (2\log(n)+\epsilon)}{2n (n^{\frac{1}{3}}-1)\epsilon^2}\le 0.
\end{align}
There exists at least one $\beta$ such that 
$g(\beta)\le 0$  if $(1+\frac{7 }{2n (n^{1/3}-1)\epsilon})^2-\frac{14 (2\log(n)+\epsilon)}{ n (n^{1/3}-1)\epsilon^2}\ge 0$, which can be ensured by the condition $\epsilon \ge \frac{7}{2n(n^{1/3}-1)} + 2\sqrt{\frac{7\log(n)}{n(n^{1/3}-1)}}$.  Furthermore,
$g(\beta )\le 0$ for all
$\beta  \in [\beta_3,\beta_4]$, where $\beta_{3,4}=\frac{1}{2}\Big(\big( 1+\frac{7 }{2n (n^{1/3}-1)\epsilon} \big)\mp \sqrt{\big( 1+\frac{7 }{2n (n^{1/3}-1)\epsilon} \big)^2-\frac{14 (2\log(n)+\epsilon)}{n \big(n^{1/3}-1\big)\epsilon^2}}  \Big)$ are two roots of $g(\beta )= 0$. 
Finally, note that 
$$ \max \bigg\{ \frac{7}{ 2n(n^{\frac{1}{3}}-1) } + 2\sqrt{ \frac{7\log(n)}{n(n^{\frac{1}{3}}-1)} }, \frac{\log(n) \big( 14\log(n) (n^{\frac{1}{3}}-1) + 162n -63 \big) }{ 9n \big( 2\log(n) (n^{\frac{1}{3}}-1) -9 \big) }  \bigg\} \le \frac{ 7(n^{\frac{1}{3}}-1) + 4\log(n) n +7 }{ 2n(n^{\frac{1}{3}}-1) }. $$
Then if   $n\ge 18$  and 
$$ \epsilon \ge   \frac{ 7(n^{\frac{1}{3}}-1) + 4\log(n) n +7 }{ 2n(n^{\frac{1}{3}}-1) }, $$
  there hold
\begin{align}\label{eq:condition-beta-1}
     \beta_3 \le \min\bigg\{  \frac{7   \log(n)}{9n \epsilon} , 1-\frac{2\log(n)}{(n^{\frac{1}{3}}-1)\epsilon} \bigg\} 
\end{align}
and 
\begin{align}\label{eq:condition-beta-2}
    \beta_3\le \beta_1 \;\text{\ if\ } \big( 1+\frac{7 }{1.34 n  \epsilon^2} \big)^2 \ge \frac{28 (2\log(n)+\epsilon)}{1.34n \epsilon^2}.
\end{align}
Conditions \eqref{eq:condition-beta-1} and \eqref{eq:condition-beta-2} ensure the existence of at least one consistent $\beta\in (0,1)$ such that \eqref{eq:condition-rdp-1},  
\eqref{eq:condition-lambda-2}, \eqref{eq:condition-lambda-3}, \eqref{eq:condition-lambda-4} and \eqref{eq:condition-rdp-2} hold, which imply that \eqref{RDPcondition1} and \eqref{RDPcondition2} hold. The proof is completed.
\end{proof}

\section{Conclusion}\label{sec:conclusion}
In this paper, we are concerned with differentially private SGD algorithms with non-smooth losses in the setting of stochastic convex optimization.  In particular, we assume that the loss function is $\ga$-H\"{o}lder smooth (i.e., the gradient is $\ga$-H\"{o}lder continuous). We systematically studied the output and gradient perturbations for  SGD and  established their privacy as well as utility guarantees.  For the output perturbation, we proved that our private SGD with $\ga$-H\"{o}lder smooth losses in a bounded $\W$ can achieve $(\gep , \gd)$-DP with the excess risk rate $   \O\Big(\frac{(d\log(1/\delta))^{1/4}\sqrt{\log(n/\delta)}}{ \sqrt{n\epsilon}}\Big)$, up to some logarithmic terms, and gradient complexity $T = \O(n^{2-\ga\over 1+\ga}+n)$, which extends the results of \cite{wu2017bolt} in the strongly-smooth case. We also established similar results for SGD algorithms with output perturbation in an unbounded domain $\W= \R^d$ with excess risk $\O\Big(\frac{\sqrt{d\log(1/\delta)}\log(n/\delta)}{ n^{\frac{2}{3+\alpha}}\gep}+ \frac{\log(n/\delta)}{n^{\frac{1}{3+\alpha}}}\Big)$, up to some logarithmic terms, which are the first-ever known results of this kind for unbounded domains. For the gradient perturbation, we show that private SGD with $\ga$-H\"{o}lder smooth losses in a bounded domain $\W$ can achieve optimal excess risk $\O\Big(\frac{\sqrt{d\log(1/\delta)}}{n\gep}+\frac{1}{\sqrt{n}}\Big)$ with gradient complexity $T = \O(n^{2-\ga\over 1+\ga}+n).$  Whether one can derive privacy and utility guarantees for gradient perturbation in an unbounded domain still remains a challenging open question to us.

\medskip 
\noindent{\bf Acknowledgement.}  This work was done while Puyu Wang was a visiting student at SUNY Albany.  The corresponding author is Yiming Ying, whose work is supported by NSF grants IIS-1816227 and IIS-2008532. The work of Hai Zhang is supported by NSFC grant U1811461.

\bibliographystyle{plain}
\bibliography{main.bib}

\begin{thebibliography}{10}

\bibitem{us-census-bureau-DP}
John~M Abowd.
\newblock The challenge of scientific reproducibility and privacy protection
  for statistical agencies.
\newblock {\em Census Scientific Advisory Committee}, 2016.

\bibitem{Bach}
Francis Bach and Eric Moulines.
\newblock Non-strongly-convex smooth stochastic approximation with convergence
  rate o (1/n).
\newblock In {\em Advances in neural information processing systems}, pages
  773--781, 2013.

\bibitem{bassily2020stability}
Raef Bassily, Vitaly Feldman, Crist{\'o}bal Guzm{\'a}n, and Kunal Talwar.
\newblock Stability of stochastic gradient descent on nonsmooth convex losses.
\newblock {\em Advances in Neural Information Processing Systems}, 33, 2020.

\bibitem{bassily2019private}
Raef Bassily, Vitaly Feldman, Kunal Talwar, and Abhradeep~Guha Thakurta.
\newblock Private stochastic convex optimization with optimal rates.
\newblock In {\em Advances in Neural Information Processing Systems}, pages
  11279--11288, 2019.

\bibitem{bousquet2008tradeoffs}
L{\'e}on Bottou and Olivier Bousquet.
\newblock The tradeoffs of large scale learning.
\newblock In {\em Advances in neural information processing systems}, pages
  161--168, 2008.

\bibitem{bousquet2002stability}
Olivier Bousquet and Andr{\'e} Elisseeff.
\newblock Stability and generalization.
\newblock {\em Journal of machine learning research}, 2(Mar):499--526, 2002.

\bibitem{bousquet2019sharper}
Olivier Bousquet, Yegor Klochkov, and Nikita Zhivotovskiy.
\newblock Sharper bounds for uniformly stable algorithms.
\newblock {\em arXiv preprint arXiv:1910.07833}, 2019.

\bibitem{carlini2019secret}
Nicholas Carlini, Chang Liu, {\'U}lfar Erlingsson, Jernej Kos, and Dawn Song.
\newblock The secret sharer: Evaluating and testing unintended memorization in
  neural networks.
\newblock In {\em 28th $\{$USENIX$\}$ Security Symposium ($\{$USENIX$\}$
  Security 19)}, pages 267--284, 2019.

\bibitem{chaudhuri2011differentially}
Kamalika Chaudhuri, Claire Monteleoni, and Anand~D Sarwate.
\newblock Differentially private empirical risk minimization.
\newblock {\em Journal of Machine Learning Research}, 12(Mar):1069--1109, 2011.

\bibitem{Dieuleveut}
Aymeric Dieuleveut and Francis Bach.
\newblock Nonparametric stochastic approximation with large step-sizes.
\newblock {\em The Annals of Statistics}, 44(4):1363--1399, 2016.

\bibitem{microsoft-DP}
Bolin Ding, Janardhan Kulkarni, and Sergey Yekhanin.
\newblock Collecting telemetry data privately.
\newblock In {\em Advances in Neural Information Processing Systems}, pages
  3571--3580, 2017.

\bibitem{dwork2009differential}
Cynthia Dwork and Jing Lei.
\newblock Differential privacy and robust statistics.
\newblock In {\em Proceedings of the forty-first annual ACM symposium on Theory
  of computing}, pages 371--380, 2009.

\bibitem{dwork2006calibrating}
Cynthia Dwork, Frank McSherry, Kobbi Nissim, and Adam Smith.
\newblock Calibrating noise to sensitivity in private data analysis.
\newblock In {\em Theory of cryptography conference}, pages 265--284. Springer,
  2006.

\bibitem{dwork2014algorithmic}
Cynthia Dwork, Aaron Roth, et~al.
\newblock The algorithmic foundations of differential privacy.
\newblock {\em Foundations and Trends{\textregistered} in Theoretical Computer
  Science}, 9(3--4):211--407, 2014.

\bibitem{google-DP}
{\'U}lfar Erlingsson, Vasyl Pihur, and Aleksandra Korolova.
\newblock Rappor: Randomized aggregatable privacy-preserving ordinal response.
\newblock In {\em Proceedings of the 2014 ACM SIGSAC conference on computer and
  communications security}, pages 1054--1067, 2014.

\bibitem{feldman2020private}
Vitaly Feldman, Tomer Koren, and Kunal Talwar.
\newblock Private stochastic convex optimization: optimal rates in linear time.
\newblock In {\em Proceedings of the 52nd Annual ACM SIGACT Symposium on Theory
  of Computing}, pages 439--449, 2020.

\bibitem{hardt2016train}
Moritz Hardt, Ben Recht, and Yoram Singer.
\newblock Train faster, generalize better: Stability of stochastic gradient
  descent.
\newblock In {\em International Conference on Machine Learning}, pages
  1225--1234, 2016.

\bibitem{hoeffding1994probability}
Wassily Hoeffding.
\newblock Probability inequalities for sums of bounded random variables.
\newblock In {\em The Collected Works of Wassily Hoeffding}, pages 409--426.
  Springer, 1994.

\bibitem{johnson2013accelerating}
Rie Johnson and Tong Zhang.
\newblock Accelerating stochastic gradient descent using predictive variance
  reduction.
\newblock In {\em Advances in neural information processing systems}, pages
  315--323, 2013.

\bibitem{lacoste2012simpler}
Simon Lacoste-Julien, Mark Schmidt, and Francis Bach.
\newblock A simpler approach to obtaining an o (1/t) convergence rate for the
  projected stochastic subgradient method.
\newblock {\em arXiv preprint arXiv:1212.2002}, 2012.

\bibitem{lei2020fine}
Yunwen Lei and Yiming Ying.
\newblock Fine-grained analysis of stability and generalization for stochastic
  gradient descent.
\newblock In {\em International Conference on Machine Learning}, pages
  5809--5819. PMLR, 2020.

\bibitem{liang2020exploring}
Zhicong Liang, Bao Wang, Quanquan Gu, Stanley Osher, and Yuan Yao.
\newblock Exploring private federated learning with laplacian smoothing.
\newblock {\em arXiv preprint arXiv:2005.00218}, 2020.

\bibitem{lin2016optimal}
Junhong Lin and Lorenzo Rosasco.
\newblock Optimal learning for multi-pass stochastic gradient methods.
\newblock In {\em Advances in Neural Information Processing Systems}, pages
  4556--4564, 2016.

\bibitem{liu2017algorithmic}
Tongliang Liu, G{\'a}bor Lugosi, Gergely Neu, and Dacheng Tao.
\newblock Algorithmic stability and hypothesis complexity.
\newblock {\em arXiv preprint arXiv:1702.08712}, 2017.

\bibitem{apple-DP}
Robert McMillan.
\newblock Apple tries to peek at user habits without violating privacy.
\newblock {\em The Wall Street Journal}, 2016.

\bibitem{mironov2017renyi}
Ilya Mironov.
\newblock R{\'e}nyi differential privacy.
\newblock In {\em 2017 IEEE 30th Computer Security Foundations Symposium
  (CSF)}, pages 263--275. IEEE, 2017.

\bibitem{Nem}
Arkadi Nemirovski, Anatoli Juditsky, Guanghui Lan, and Alexander Shapiro.
\newblock Robust stochastic approximation approach to stochastic programming.
\newblock {\em SIAM Journal on optimization}, 19(4):1574--1609, 2009.

\bibitem{Orabona}
Francesco Orabona.
\newblock Simultaneous model selection and optimization through parameter-free
  stochastic learning.
\newblock In {\em Advances in Neural Information Processing Systems}, pages
  1116--1124, 2014.

\bibitem{rakhlin2012making}
Alexander Rakhlin, Ohad Shamir, and Karthik Sridharan.
\newblock Making gradient descent optimal for strongly convex stochastic
  optimization.
\newblock In {\em Proceedings of the 29th International Conference on Machine
  Learning}, pages 449--456, 2012.

\bibitem{shamir2013stochastic}
Ohad Shamir and Tong Zhang.
\newblock Stochastic gradient descent for non-smooth optimization: Convergence
  results and optimal averaging schemes.
\newblock In {\em International Conference on Machine Learning}, pages 71--79,
  2013.

\bibitem{shokri2017membership}
Reza Shokri, Marco Stronati, Congzheng Song, and Vitaly Shmatikov.
\newblock Membership inference attacks against machine learning models.
\newblock In {\em 2017 IEEE Symposium on Security and Privacy (SP)}, pages
  3--18. IEEE, 2017.

\bibitem{smale2006online}
Steve Smale and Yuan Yao.
\newblock Online learning algorithms.
\newblock {\em Foundations of Computational Mathematics}, 6(2):145--170, 2006.

\bibitem{song2015learning}
Shuang Song, Kamalika Chaudhuri, and Anand Sarwate.
\newblock Learning from data with heterogeneous noise using sgd.
\newblock In {\em Artificial Intelligence and Statistics}, pages 894--902,
  2015.

\bibitem{wainwright2019high}
Martin~J Wainwright.
\newblock {\em High-dimensional statistics: A non-asymptotic viewpoint},
  volume~48.
\newblock Cambridge University Press, 2019.

\bibitem{wu2017bolt}
Xi~Wu, Fengan Li, Arun Kumar, Kamalika Chaudhuri, Somesh Jha, and Jeffrey
  Naughton.
\newblock Bolt-on differential privacy for scalable stochastic gradient
  descent-based analytics.
\newblock In {\em Proceedings of the 2017 ACM International Conference on
  Management of Data}, pages 1307--1322, 2017.

\bibitem{ying2008online}
Yiming Ying and Massimiliano Pontil.
\newblock Online gradient descent learning algorithms.
\newblock {\em Foundations of Computational Mathematics}, 8(5):561--596, 2008.

\bibitem{ying2017unregularized}
Yiming Ying and Ding-Xuan Zhou.
\newblock Unregularized online learning algorithms with general loss functions.
\newblock {\em Applied and Computational Harmonic Analysis}, 42(2):224--244,
  2017.

\end{thebibliography}

\bigskip

\noindent{\bf Appendix: Proof of Lemma~\ref{lem:generror-high-probability}}
 
\medskip 

In the appendix, we present the proof of  Lemma~\ref{lem:generror-high-probability}. To this aim, we introduce the following lemma.
\begin{lemma}\label{lem:highpro-gener-1}
	Suppose $\ell$ is nonnegative, convex and $\alpha$-H\"older smooth.  Let $\A$ be a randomized algorithm with $ \sup_{S\simeq S'}\delta_\A(S,S')\le \Delta_\A$. Suppose the output of $\A$ is bounded by $G>0$ and let $M_0=\sup_{z\in \Z} \ell(0,z)$, $M=\sup_{z\in\Z} \|\partial \ell(0, z)\|_2 $. Then for any $\gamma \in (0,1)$, there holds 
		$$  \mathbb{P}_{\S \sim \D^n, \A}\left[|\cR(\mathcal{A(S)})-\cR_S(\mathcal{A(S)})|\ge c\bigg((M+ LG^\alpha  ) \Delta_\A \log(n)\log(1/{\gamma})+ \big( M_0+ (M+LG^\alpha)G   \big)\sqrt{n^{-1}\log(1/\gamma)} \bigg)\right] \le  \gamma.  $$
 \end{lemma}

\begin{proof}
	By the convexity of $\ell$ and the definition of $\alpha$-H\"older smoothness, we have for any $S$ and $S'$,
	\begin{align}\label{eq:bound-ell}
		\ell(\A(S),z)&\le \sup_{z\in \Z} \ell(0,z)+ 	\langle \partial \ell(\A(S), z) , \A(S)   \rangle \le M_0+ \|\partial \ell(\A(S), z)\|_2\|\A(S)\|_2\nonumber\\
		&\le M_0+ (M+ L\|\A(S)\|^\alpha_2 )\|\A(S)\|_2\le M_0+ (M+LG^\alpha)G  
		\end{align}
	and
	\begin{align*}
		\sup_{z\in\Z} | \ell(\A(S),z)-\ell(\A(S'),z) |&\le  \max\big\{\| \partial \ell(\A(S), z)  \|_2,  \|\partial \ell(\A(S'), z)\|_2  \big\} \|\A(S)-\A(S')\|_2\nonumber\\
		&\le (M+ LG^\alpha  ) \|\A(S)-\A(S')\|_2.  
	\end{align*} 
	Note  $ \sup_{S\simeq S'}\delta_\A(S,S')\le \Delta_\A$ and $\delta_{\A}(S,S')=\|\A(S)-\A(S')\|_2$. Then for any neighboring datasets $S\simeq S'$, we have
	\begin{align}\label{eq:bound-ell-difference}
		\sup_{z\in\Z} | \ell(\A(S),z)-\ell(\A(S'),z) | 
		&\le (M+ LG^\alpha ) \Delta_\A  .
	\end{align}  
 Combining Eq. \eqref{eq:bound-ell}, Eq. \eqref{eq:bound-ell-difference} and Corollary~8 in \cite{bousquet2019sharper} together, we derive the following probabilistic inequality
 $$ \mathbb{P}_{\S \sim \D^n, \A}\left[|\cR(\mathcal{A(S)})-\cR_S(\mathcal{A(S)})|\ge c\bigg((M+ LG^\alpha  ) \Delta_\A \log(n)\log(1/{\gamma})+ \big( M_0+ (M+LG^\alpha)G \big)  \sqrt{n^{-1}\log(1/\gamma)} \bigg)\right] \le  \gamma.  
 $$
 The proof is completed.
\end{proof}



\begin{proof}[Proof of Lemma~\ref{lem:generror-high-probability}]
Let $E_1=\{ \A : \sup_{S\simeq S'}\|\A(S)-\A(S')\|_2\ge \Delta_\A \}$ and $E_2=\Big\{ (S, \A):| \cR(\mathcal{A(S)})-\cR_S(\mathcal{A(S)})|\ge c\Big((M+ LG^\alpha  ) \Delta_\A \log(n)\log(1/{\gamma})+\big( M_0+ (M+LG^\alpha)G\big)   \sqrt{n^{-1}\log(1/\gamma)} \Big) \Big\}$. Then by the assumption we have $\mathbb{P}_\A[ \A \in E_1 ]\le \gamma_0$. Further, according to Lemma~\ref{lem:highpro-gener-1}, for any $\gamma \in (0,1)$, we have $\mathbb{P}_{S,\A}[(S, \A)  \in E_2 \cap \A \notin E_1 ]\le \gamma $. Therefore,  
\begin{align*}
	 \mathbb{P}_{S,\A}[(S, \A) \in E_2   ]&= \mathbb{P}_{S,\A}[(S, \A)  \in E_2 \cap    \A \in E_1 ] +  \mathbb{P}_{S,\A}[(S, \A)  \in E_2 \cap    \A \notin E_1 ] \nonumber\\
	 &\le \mathbb{P}[\A \in E_1]  +  \mathbb{P}_{S,\A}[(S, \A)  \in E_2 \cap    \A \notin E_1 ]\le \gamma_0+\gamma. 
\end{align*} 
The proof is completed. 
\end{proof}

\end{document}